%% file: main.tex
\documentclass[11pt]{article}
\usepackage[english]{babel}
\usepackage{graphicx}
\usepackage{xcolor}
\usepackage{framed}
\usepackage[normalem]{ulem}
\usepackage{amsmath}
\usepackage{amsthm,thmtools,mathtools}
\usepackage{amssymb}
\usepackage{amsfonts}
\usepackage{enumerate}
\usepackage{algorithmic} 
\usepackage[ruled, vlined,noend]{algorithm2e}
\usepackage{wrapfig} 
\usepackage{caption}
\usepackage{subcaption}
\usepackage{fullpage}
\usepackage[margin=1in]{geometry}
\usepackage[
  pagebackref,
  colorlinks=true,
  urlcolor=blue,
  linkcolor=black,
  citecolor=blue,
]{hyperref}
\usepackage[nameinlink]{cleveref}
\theoremstyle{definition}
\newtheorem{theorem}{Theorem}
\newtheorem{observation}{Observation}
\newtheorem{corollary}{Corollary}
\newtheorem{lemma}{Lemma}
\newtheorem{proposition}{Proposition}
\newtheorem{definition}{Definition}
\newtheorem{example}{Example}

\newtheorem{extension}{Extension}

\DeclareMathOperator*{\argmax}{arg\,max} 

\renewcommand{\vec}[1]{\mathbf{#1}}

\SetCommentSty{mycommfont}

\begin{document}

\title{The Strategic Perceptron}
\author{Saba Ahmadi\thanks{University of Maryland. Email: \texttt{saba@umd.edu}. Research initiated during author's visit to Northwestern University. Author was supported in part by a research award from Amazon, and NSF CCF-1733556. Part of the research was done when author was visiting Toyota Technological Institute at Chicago.} \and%
Hedyeh Beyhaghi\thanks{Toyota Technological Institute at Chicago. Email: \texttt{hedyeh@ttic.edu}.This work was done in part while the author was a Postdoctoral Researcher at Northwestern University, supported in part by the National Science Foundation under grant CCF-1618502.} \and%
Avrim Blum\thanks{Toyota Technological Institute at Chicago. Email: \texttt{avrim@ttic.edu}. This work was supported in part by the National Science Foundation under grants CCF-1815011 and CCF-1733556.} \and%
Keziah Naggita\thanks{Toyota Technological Institute at Chicago. Email: \texttt{knaggita@ttic.edu}. This work was supported in part by the National Science Foundation under grant CCF-1815011.}}
    \date{}
\maketitle

\begin{abstract}
The classical \emph{Perceptron algorithm} provides a simple and elegant procedure for learning a linear classifier. In each step, the algorithm observes the sample’s \emph{position} and \emph{label} and updates the current predictor accordingly if it makes a mistake. However, in presence of strategic agents that desire to be classified as positive and that are able to modify their position by a limited amount, the classifier may not be able to observe the true position of agents but rather a position where the agent pretends to be. Unlike the original setting with perfect knowledge of positions, in this situation the Perceptron algorithm fails to achieve its guarantees, and we illustrate examples with the predictor oscillating between two solutions forever, making an unbounded number of mistakes even though a perfect large-margin linear classifier exists. Our main contribution is providing a modified Perceptron-style algorithm which makes a bounded number of mistakes in presence of strategic agents with both $\ell_2$ and weighted $\ell_1$ manipulation costs.  In our baseline model, knowledge of the manipulation costs (i.e., the extent to which an agent may manipulate) is assumed.  In our most general model, we relax this assumption and provide an algorithm which learns and refines both the classifier and its cost estimates to achieve good mistake bounds even when manipulation costs are unknown.
\end{abstract}

\newpage
\input{introduction}
\input{prelim}

\input{L2}

\input{L1}

\input{unknown_cost}
\input{different_cost}

\input{non_zero_bias}
\input{conclusions}
\bibliographystyle{plain}
\newpage
\footnotesize{
\bibliography{ref}
}
\newpage
\normalsize
\appendix
\input{supplementary}
\end{document}

%% file: introduction.tex
\section{Introduction}

In machine learning,  \emph{strategic classification} deals with the problem of learning a classifier when the learner relies on data that is provided by strategic agents~\cite{10.1145/2020408.2020495, Hardt2016}. For example, consider deciding eligibility of individuals for employment or education. In order to be considered eligible, individuals may engage in activities that do not truly change their qualifications, but affect the decision made. In the aforementioned settings, these activities include job or college applicants carefully crafting their application materials and investing in interview or test preparations. In these scenarios, by using information about the classifier, individuals alter their features artificially by a limited amount to achieve their desirable outcome.

Strategic classification is particularly challenging in the online setting, where data points arrive in an arbitrary sequence, because the way that points manipulate may depend (in a discontinuous way) on the current classifier, and there is no useful source of unmanipulated data.
More specifically, consider a standard online learning setting as follows.  Individuals arrive one at a time, and based on the individual's features, the classifier predicts the individual as positive or negative. The learner is then told the correct classification and may update its classifier for the next round. The learner's goal is to minimize the number of mistakes made. Performing the same procedure in the strategic setting brings in several challenges. First, since the learner does not observe the \emph{true} features, the update is done based on the individual's \emph{manipulated} features. Therefore, at each point in time, the current classifier is built from manipulated data the learner has observed in the past. Second, each individual reacts to the \emph{current} classifier. This means that the individuals' behaviors change over time and may be different from behavior of previous individuals with similar features.  Moreover, because data arrives in an arbitrary order, there is no way to collect a ``representative sample'' of unmanipulated data by, say, classifying all examples as negative for an initial period. Finally, manipulation behavior may be a discontinuous function of the classifier's parameters: if an individual's cost to manipulate is slightly less than the benefit of being classified as positive then it will do so, but if it is slightly greater then it will not.  Due of these issues, as we will show, standard learning algorithms that would make a limited number of mistakes in non-strategic settings may end up cycling and making unbounded number of mistakes; even if there exists a perfect classifier they may not find one.

Another challenge in online strategic classification is when the learner is unaware of the manipulation costs, which determine the extent to which agents will manipulate their features to achieve a positive classification. In this case, on top of estimating the individuals' real attributes based on the observed data, the learner also needs to estimate the costs. Unreasonable estimate of costs may lead to poor performance by the learner as the learner may not be able to distinguish if a classification mistake is due to an improper classifier or improper estimate of costs. This failure to distinguish correctly may lead to deterioration of the classifier and divergence from the optimal solution.

We study an online linear classification problem when the individuals are strategic.
To isolate the effect of manipulation, we focus on finding a linear classifier when the unmanipulated data is linearly separable; i.e., the feature space is divided into two half spaces: with \emph{positive} data points in one and \emph{negative} data points in the other, and a nonzero margin between them. 
When individuals can manipulate, in each step, the arriving individual wishes to be classified positively. If the individual's feature vector $\vec{z}$ is not classified as positive with the true attributes, they may choose to suffer a cost and pretend to have a feature vector $\vec{x}$. More specifically, we consider utility-maximizing individuals, where utility is defined as value minus cost, who receive value 1 for being classified as positive and 0 for being classified as negative.
 We then consider two classes of cost functions: $\ell_2$ costs (where cost is proportional to the Euclidean distance moved) and weighted $\ell_1$ costs (where the cost of reaching a destination is the sum of separate costs paid in each coordinate direction).  The $\ell_2$ case represents settings where individuals when manipulating can take actions that affect multiple attributes. The $\ell_1$ case represents settings where there is a specific action associated with each attribute. Note that in both cases, even though the \emph{unmanipulated} data is linearly separable, the observed \emph{manipulated} data points may no longer be separable.

{\paragraph{Our Techniques and Results}
The main contribution of this paper is solving the problem of online learning of linear separators in the strategic setting, making a bounded number of mistakes when the unmanipulated data is linearly separable by a nonzero margin.  To do this, we build on and adapt the classic 
Perceptron algorithm~\cite{rosenblatt1958perceptron}, redesigning it to work in various strategic settings. This classic algorithm makes a bounded number of mistakes in the nonstrategic case when positive and negative data points are linearly separable.  However, as mentioned earlier, in the strategic case it may cycle indefinitely (much like gradient descent for finding a Nash equilibrium) and make an unbounded number of mistakes; see \Cref{ex:perceptron_L2_fails,ex:perceptron_L2_fails_bias_term}. Our main technique is to carefully design \emph{surrogate} data points and feed them as the observed data to the algorithm. The role of the surrogate is to ensure that the algorithm is able to make positive progress each time it makes a mistake; however, defining it requires extra care.  In particular, while it is not hard to show we can compute the {\em direction} that data points may have manipulated in, we can never be sure exactly how far (and we are particularly interested in the case that the amount by which data points can manipulate is large compared to the margin of separation).  
Another adaptation is to use a positive threshold for the dot product with the classifier’s weight vector for a point to be classified positive.}

{Making use of the Perceptron algorithm, surrogate data points, and a positive linear threshold is central in all the algorithms designed in this paper. However, additional ideas are needed to handle subtleties of each specific setting. 
For example, for weighted $\ell_1$ costs we need to take extra steps to make the manipulation direction unique and in line with the true classifier’s weight vector, and in the unknown costs setting, we need to distinguish if the cost estimates are above or below the true costs. 
Another case is when the separating hyperplane does not cross the origin. In this case, the classic approach is to just add a fake coordinate in which each example has value 1, and then apply the Perceptron algorithm to those extended data points.  However, when data is given by strategic agents, this reduction breaks down and we need to apply different ideas. There are also some results that hold for the non-strategic case that we do not know how to achieve, such as obtaining a mistake bound proportional to the hinge-loss of the best separator when data is not perfectly separable; for this setting we show examples where our algorithm fails and propose it as an open problem.}

The main contributions of this paper are:
\begin{itemize}
\item[-] We give an online learning algorithm robust to manipulation that finds a linear classifier in a bounded number of mistakes with the knowledge of costs. The number of mistakes is not much larger than the standard Perceptron bound in the non-strategic case {for $\ell_2$ costs and is reasonably bounded in other settings as well}, see \Cref{thm:num_updates,thm:num-updates-mutiple-direction-manipulation,thm:bias}.
\item[-] {We give an online learning algorithm that generalizes the previous algorithm to unknown costs with a bounded number of mistakes. See \Cref{thm:l2-unknown-cost}.}
\item[-] {We generalize the algorithm for known $\ell_2$ costs to the case of heterogeneous agents whose utility functions differ by a limited amount and give an online learning algorithm with bounded number of mistakes. See \Cref{cor:different_cost}}.
\end{itemize}

\paragraph{Related Work}

The first studies on strategic classification focused on the offline setting; i.e., where the agents' true features come from a distribution.  Br\"{u}ckner and Scheffer~\cite{10.1145/2020408.2020495} and later Hardt et al.~\cite{Hardt2016}, formalized the strategic classification problem as a Stackelberg competition between a learner and an agent. They assume the learner has access to the distribution of agents true features and their cost functions; and use this information to design near-optimal classifiers.

Dong et al.~\cite{10.1145/3219166.3219193} initiated the study of strategic classification in the online setting where the learner does not know the distribution of agents' true features or their cost functions. A key difference between \cite{10.1145/3219166.3219193} and this paper is the assumption on the objective of the agents: we consider agents that wish to be classified as positive, whereas \cite{10.1145/3219166.3219193} considers agents that wish to increase their dot-product with the hypothesis vector no matter how they are classified. {Based on this assumption, in \cite{10.1145/3219166.3219193} the agents’ behaviors are continuous in the hypothesis vector. However, in our model, a small change in the hypothesis vector can cause a drastic (discontinuous) change in agents' behavior. More particularly, as a consequence of agents’ objective and utility structure, each agent can manipulate by a limited amount. If the classification hyperplane is closer than this amount, the agent would manipulate to be classified as positive; however, if it is slightly farther, the agent stays stationary. This discontinuity in the agent’s behavior is common in mechanism design and occurs in other problems such as pricing and auction design.}

{Chen et al.~\cite{chen2020learning} also study an online learning problem where agents can manipulate by a bounded distance.
However, while there are similarities between the setting studied in \cite{chen2020learning} and our $\ell_2$ cost model, explained in more detail in \Cref{sec:model}, \cite{chen2020learning} does not consider a fixed utility model and instead considers a regret term that is worst-case over agents that can manipulate by some bounded distance. As a result, their regret term may be arbitrarily high when the observed positions of positive and negative data points are inseparable, even if the unmanipulated points are linearly separable. Our algorithms, in contrast, can handle this inseparability during the learning procedure and make a bounded number of mistakes.}

{The goal of the papers mentioned so far is accuracy, or minimizing loss. There are also papers that consider other objectives.} Hu et al.~\cite{Hu:2019:} focus on a fairness objective and raise the issue that different populations of agents may have different manipulation costs. Braverman and Garg~\cite{braverman2020role}, by introducing noise in their classification, design algorithms where agents with different costs are better off not manipulating which tackles the fairness issue. Milli et al.~\cite{Milli2018TheSC} state that the accuracy that strategic classification seeks leads to a raised bar for agents who naturally are qualified and puts a burden on them to prove themselves.
Kleinberg and Raghavan~\cite{Kleinberg2018HowDC}, Haghtalab et al.~\cite{haghtalab2020maximizing}, Alon et al.~\cite{alon2020multiagent}, Bechavod et al.~\cite{bechavod2020causal}, Shavit et al.~\cite{shavit2020learning}, and Miller et al.~\cite{miller2019strategic}
focus on models in which the policy maker is interested in choosing a rule which incentivizes agent(s) to invest their effort into features that truly improve their qualification.
\paragraph{Organization of the Paper.} \Cref{sec:model} introduces the model and provides examples where the original Perceptron algorithm makes an unbounded number of mistakes. \Cref{sec:strategic-Perceptron,sec:generalization} study the case where the cost of manipulation is known:  \Cref{sec:strategic-Perceptron} focuses on $\ell_2$ costs and \Cref{sec:generalization} on weighted $\ell_1$ costs. \Cref{sec:strategic-Perceptron-unknown-alpha} studies the unknown costs model. \Cref{sec:different_costs} studies the generalization of known manipulation costs to heterogeneous agents that have slightly different costs. \Cref{sec:bias} extends the results of  \Cref{sec:strategic-Perceptron,sec:generalization} where the separator of the unmanipulated data points crosses the origin to the general case. Finally, conclusions and some open problems are presented in \Cref{sec:conclusion}.

%% file: prelim.tex
\section{Model and Preliminaries}\label{sec:model}
In \Cref{sec:model_model}, we formally define our model. In  \Cref{sec:model_non_strategic}, we overview the non-strategic setting. In \Cref{sec:model_perceptron_fail}, we provide examples where the original Perceptron algorithm makes an unbounded number of mistakes.

\subsection{Model}\label{sec:model_model}
We study an online classification problem in which a series of examples in $\mathbb{R}^d$ arrive one at a time.  We think of examples as corresponding to $d$ observable features of individuals who wish to be classified as positive.  
They have the ability to manipulate their observable features at some cost.
Let $\vec{z}_t$ denote the $t^{th}$ example before manipulation, and $\vec{x}_t$ denote the observed $t^{th}$ example.
We assume there exists a vector $\vec{w}^*$, such that for each unmanipulated positive example $\vec{z}_t$ we have $\vec{z}_t^T\vec{w^*}\geq 1$, and for each unmanipulated negative example $\vec{z}_t$ we have $\vec{z}_t^T\vec{w^*}\leq -1$; i.e., a linear separator of margin $\gamma = 1/|\vec{w^*}|$. 
We use $|\vec{w}|$ to indicate the $\ell_2$ norm of $\vec{w}$.

We assume individuals are utility maximizers, where utility is defined as value minus cost. Individuals have value $1$ for being classified as positive, and $0$ for being classified as negative.
More formally, an agent with true coordinates $\vec{z}_t$ will move to $\vec{x}_t = \argmax_x [value(\vec{x}) - cost(\vec{z}_t,\vec{x})]$ where $value(\vec{x})=1$ if $\vec{x}$ is classified positive by the current classifier and $value(\vec{x})=0$ if $\vec{x}$ is classified negative, and $cost(\vec{z}_t,\vec{x})$ refers to the cost of manipulation from $\vec{z}_t$ to $\vec{x}$. This implies if the agent can manipulate their features at cost at most 1 to change their classification from negative to positive, then they will do so in the cheapest way possible, otherwise they will not.

We consider two settings for cost of manipulation. In the first setting, $cost(\vec{z}_t,\vec{x})$ is proportional to the $\ell_2$ distance of the two points $\vec{z}_t$ and $\vec{x}$; i.e., $cost(\vec{z}_t,\vec{x})=c\sqrt{\sum_{i=1}^d (\vec{x}_{i}-\vec{z}_{t,i})^2}$, where $c$ is the cost per unit of movement. 
 We define $\alpha=1/c$ as the maximum amount data points would be willing to move to achieve a positive classification.\footnote{For convenience we assume that if an agent is indifferent, i.e., its distance to the decision boundary is exactly $\alpha$, then it will maniplulate.  Note that Chen et al.~\cite{chen2020learning} also consider a model where individuals can move in a ball of fixed radius from their real position. However, they do not focus on a specific utility model.}
We assume $0\leq \alpha \leq R$ where $R=\max_t |\vec{z}_t|$. In the second setting, $cost(\vec{z}_t,\vec{x})$ is a weighted $\ell_1$ metric, such that $cost(\vec{z}_t,\vec{x})=\sum_{i=1}^d c_i|\vec{x}_{i}-\vec{z}_{t,i}|$. Similarly we define $\alpha_i=1/c_i$ as the maximum amount data points would move along the $i^{th}$ coordinate vector $\vec{e}_i$, where $0\leq \alpha_i \leq R$. 
We consider both scenarios of known and unknown costs. In the unknown $\ell_2$ costs we don't assume knowledge of $c$, and in unknown weighted $\ell_1$ costs we do not assume knowledge of $c_1, \ldots, c_d$.

\subsubsection*{Generalizations} For the majority of the paper we study the above model. In \Cref{sec:different_costs,sec:bias}, we then present and analyze several generalizations. \Cref{sec:different_costs} studies heterogeneous agents with $\ell_2$ costs where the costs per unit of movement (defined previously as $c$) for agents are slightly different. More particularly, we study the case where the maximum amount the agent arriving at time $t$ can move, $\alpha_t$ (i.e. $1/c_t$, where $c_t$ is the cost per unit of movement for agent $t$), is in the interval $[\alpha_{\min}, \alpha_{\max}]$ where $0\leq \alpha_{\max}-\alpha_{\min}\leq \gamma/2$. 
The algorithm does not have access to $\alpha_t$ but knows the interval. \Cref{sec:bias} studies the case where the separating hyperplane of the unmanipulated data does not cross the origin. More particularly, there exists a separator $\vec{z}^T\vec{w^*}+b = 0$, such that for a positive example $\vec{z}_t$, $\vec{z}_t^T\vec{w^*}+b\geq 1$ and for a negative example $\vec{z}_t$, $\vec{w^*}^T\vec{z}_t+b \leq -1$.

\subsection{Non-Strategic Setting and the Perceptron Algorithm}\label{sec:model_non_strategic}
As a reminder for the reader we provide the classical Perceptron algorithm here. This algorithm classifies all points with $\vec{x}_t^T\vec{w} \geq 0$ as positive, and the rest as negative; updating $\vec{w}$ when it makes a mistake. The total number of mistakes made by the algorithm is upper bounded by $R^2 |\vec{w^*}|^2$.

\begin{algorithm}
$\vec{w} \leftarrow \vec{0}$\;
\For{$t=1, 2, \cdots$}{
    Given example $\vec{x_t}$,
        predict $sgn(\vec{x}_t^T\vec{w})$\;
        \If{the prediction was a mistake }{
            \lIf{$\vec{x_t}$ was $+$}{
                $\vec{w} \leftarrow \vec{w}+\vec{x}_t$%
                }
            \lIf{$\vec{x}_t$ was $-$}
            { $\vec{w} \leftarrow \vec{w}-\vec{x}_t$%
            }
        }
}
\caption{Perceptron Algorithm}
\label{alg:Perceptron}
\end{algorithm}

\begin{extension}[Perceptron with separator not crossing the origin]\label{ext:bias}
A classic extension of
\Cref{alg:Perceptron} to the case where examples are linearly separable, but not by a separator passing through the origin, is to create an extra ``fake" coordinate.  Specifically, assume there exists a separator $\vec{x}^T\vec{w^*}+b = 0$, such that for a positive example $\vec{x}_t$, $\vec{x}_t^T\vec{w^*}+b\geq 1$ and for a negative example $\vec{x}_t$, $\vec{w^*}^T\vec{x}_t+b \leq -1$.  Then \Cref{alg:Perceptron} is extended by adding an extra coordinate of value $1$ to each example $\vec{x}_t$, replacing $\vec{x}_t$ with $(\vec{x}_t,1)$. The bias term $b$ is absorbed into $\vec{w^*}$ by adding an additional coordinate to $\vec{w^*}$, i.e. replacing $\vec{w^*}$ with $(\vec{w^*},b)$. Now, for the positive examples, $\vec{x}_t^T\vec{w^*}\geq 1$, and for the negative examples $\vec{x}_t^T\vec{w^*}\leq -1$, and \Cref{alg:Perceptron} can be used as before. 
\end{extension}

\subsection{Failure of the Perceptron Algorithm in Strategic Settings}
\label{sec:model_perceptron_fail}

The Perceptron algorithm may make unbounded number of mistakes in the models considered in this paper even when a perfect classifier exists. The following example illustrates this in a setting with $\ell_2$ cost.

\begin{example}
\label{ex:perceptron_L2_fails}
 Consider three examples $A=(-1,0)$, $B=(0,-1)$, and $C=(-0.5, -1)$ where $A$ is negative, $B$ is positive, and $C$ is negative. Suppose that $\alpha=0.5$. The following scenario of arrival of these examples makes the standard Perceptron algorithm (\Cref{alg:Perceptron}) cycle between two classifiers and make an unbounded number of mistakes. Suppose $A$ is the first example to arrive, then individuals $B$ and $C$ arrive respectively and repeatedly. After arrival of $A$, $\vec{w}=(1,0)$. $B$ does not need to manipulate as it is classified positive with the current classifier. However $C$ manipulates to point $(0,-1)$ and the algorithm mistakenly classifies it as positive. As a consequence,  $\vec{w}$ will be updated to $(1,0)-(0,-1)=(1,1)$. With the new classifier, $B$ cannot manipulate to be classified positive because it has distance $\sqrt{2}/2$ from the decision boundary. So, $B$ is misclassified as negative, causing an update to $\vec{w}=(1,1)+(0,-1)=(1,0)$ and the scenario repeats. \end{example}

\begin{figure}[ht!]
    \centering
    \begin{subfigure}[b]{0.46\textwidth}
        \centering
        \includegraphics[width=\textwidth]{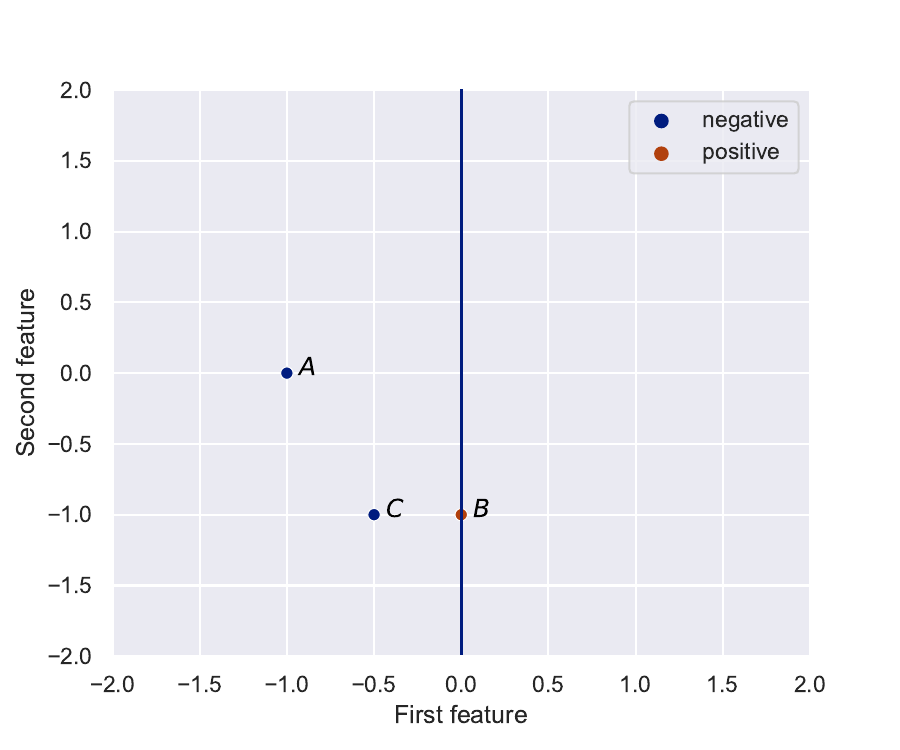}
        \caption{Classic Perceptron  correctly classifies  $B$ after update of $\vec{w}$ to $(1,0)$. However, $C$ now manipulates to the same location as $B$, which will cause a mistake and an update.}
        \label{fig:cp_step11}
    \end{subfigure}
    \hfill
    \begin{subfigure}[b]{0.46\textwidth}
        \centering
        \includegraphics[width=\textwidth]{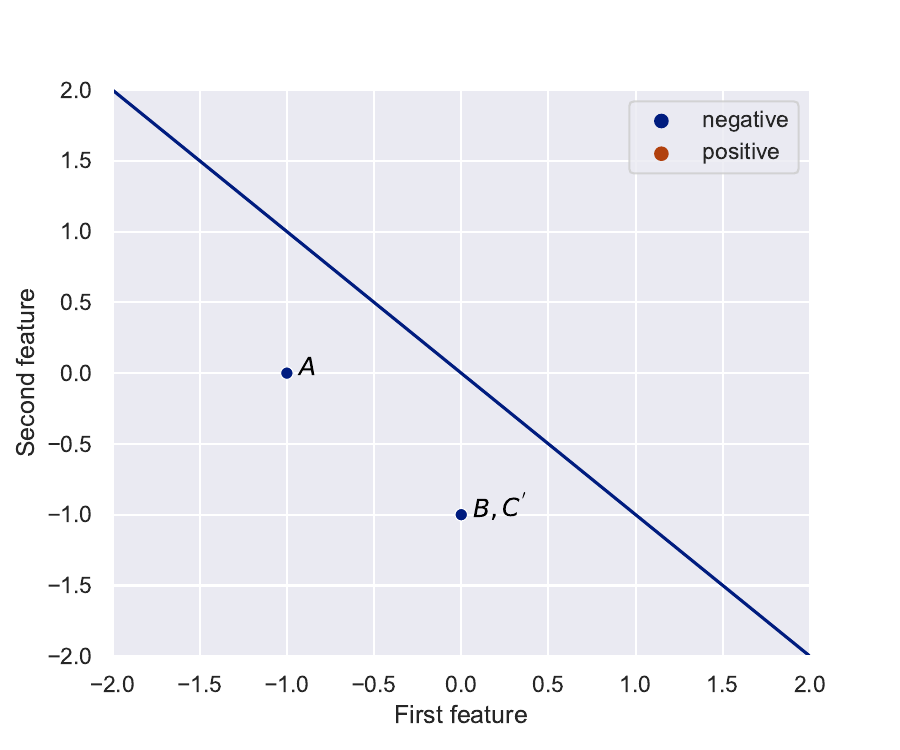}
        \caption{After $C$ manipulates to $C'=(0, -1)$, it is misclassified as positive and $\vec{w}$ is updated to $(1,1)$. $B$ cannot manipulate with current classifier because it is at distance $\sqrt{2}/2$ from the boundary.}
        \label{fig:cp_step21}
    \end{subfigure}
    \caption{{ \Cref{ex:perceptron_L2_fails} shows the classic Perceptron algorithm can make an unbounded number of mistakes} {in the strategic setting.}}
\end{figure}

Note that \Cref{ex:perceptron_L2_fails} shows that the standard Perceptron algorithm can fail even if there exists a classifier that is perfect in the presence of manipulation. In this example, the
classifier given by $\vec{w}=(1,0.5)$  works perfectly for the three points as $B$ can manipulate to be classified positive but $A$ and $C$ cannot. The main reason the algorithm fails despite existence of a perfect classifier, is that the behavior of individuals \emph{depends on the classifier} we are currently using and this can cause the algorithm to cycle indefinitely.

The failure of the Perceptron algorithm is not restricted to the $\ell_2$ costs model. \Cref{ex:perceptron_L2_fails} with $\alpha = (0.6, 0)$ makes an unbounded number of mistakes in the $\ell_1$ costs model as well.

The Perceptron algorithm as described above uses a threshold of 0.  One may wonder if the usual extension to non-zero thresholds 
(\Cref{ext:bias})
might solve the strategic learning problem.  In particular, any linearly separable dataset is still linearly separable in the presence of manipulation, by simply shifting the target separator by $\alpha$. However, the example below shows that this extension also fails when the data points are strategic.

\begin{example}
\label{ex:perceptron_L2_fails_bias_term}
{
Consider three examples $A=(-1,0)$, $B=(1,0)$, and $C=(0.5,0)$ where $A$ and $C$ are negative and $B$ is positive. Let $\alpha=0.5$. Suppose $A$ is the first example to arrive, then individuals $B$ and $C$ arrive respectively and repeatedly. After arrival of $A$, the separator is $1x_1 + 0x_2 - 1 \geq 0$. $B$ does not need to manipulate as it is classified positive with the current classifier as shown in \Cref{fig:cp_step1}. However $C=(0.5,0)$ manipulates to $(1, 0)$ and the algorithm mistakenly classifies it as positive as shown in \Cref{fig:cp_step2}. As a consequence, the separator is updated to $0x_1 + 0x_2 - 2 \geq 0$. With the new classifier, $B$ is misclassified as negative but does not manipulate.  The separator is then updated to $1x_1 + 0x_2 - 1 \geq 0$, and the process repeats indefinitely. Therefore the classifier keeps cycling and never correctly classifies the data even when a linear separator exists. 
\begin{figure}[ht!]
    \centering
    \begin{subfigure}[b]{0.46\textwidth}
        \centering
        \includegraphics[width=\textwidth]{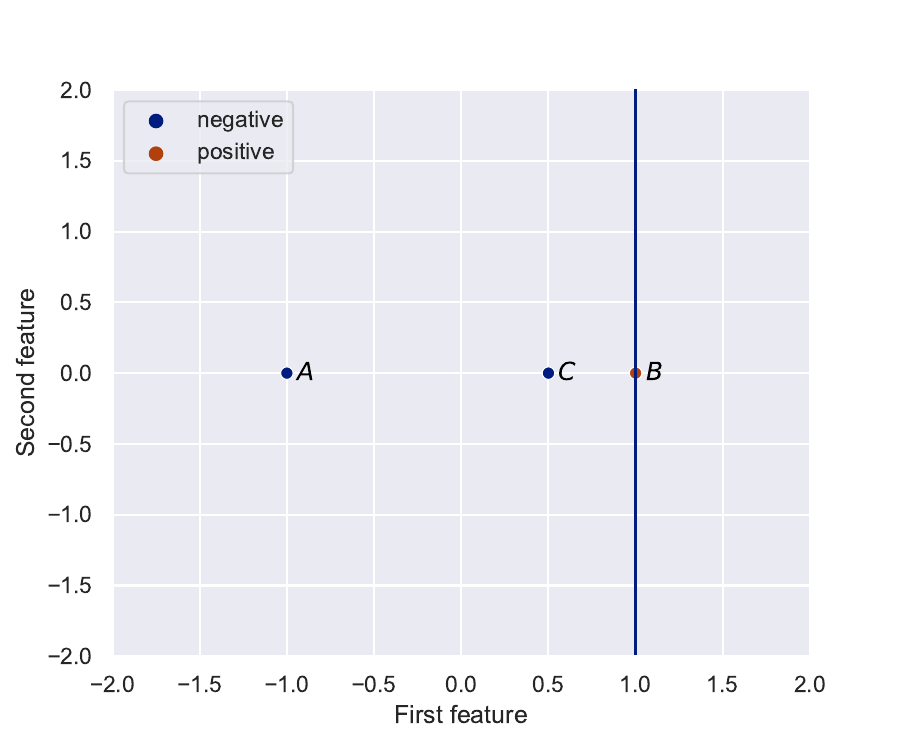}
        \caption{The Perceptron algorithm updates $\vec{w}$ after misclassifying $A$. When $B$ arrives it correctly classifies it.}
        \label{fig:cp_step1}
    \end{subfigure}
    \hfill
    \begin{subfigure}[b]{0.46\textwidth}
        \centering
        \includegraphics[width=\textwidth]{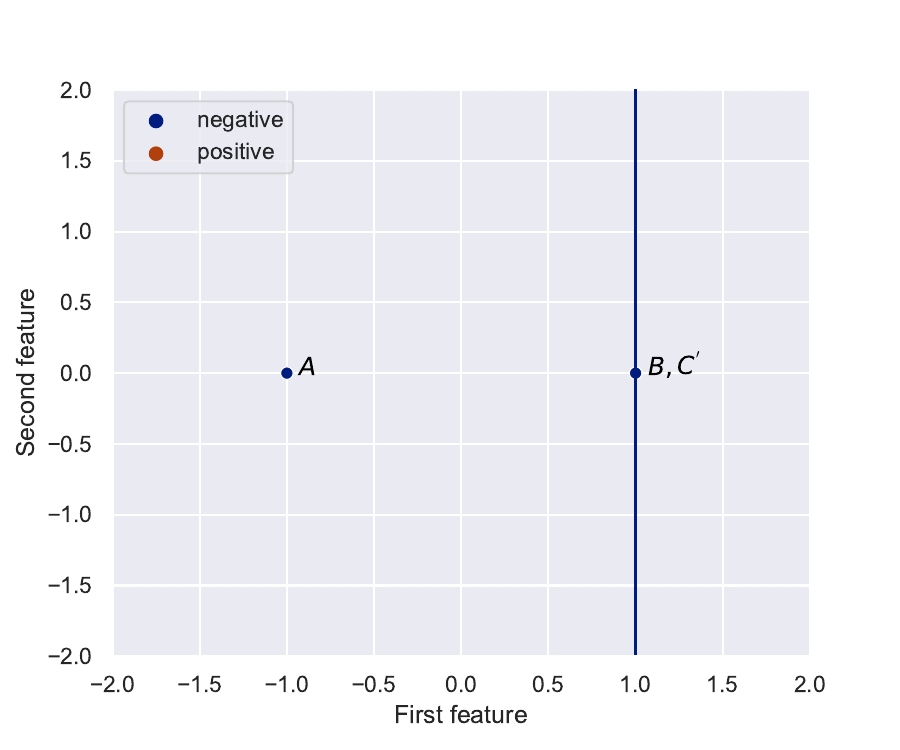}
        \caption{After $C$ manipulates to $C'=(1, 0)$, it is now misclassified as positive by the current classifier and $\vec{w}$  and $bias$ are updated accordingly.}
        \label{fig:cp_step2}
    \end{subfigure}
    \caption{{\Cref{ex:perceptron_L2_fails_bias_term} shows the non-zero threshold Perceptron algorithm can make an unbounded number of mistakes} {in the strategic setting.}}
\end{figure}
}
\end{example}

%% file: L2.tex
\section{Known $\ell_2$ Costs}
\label{sec:strategic-Perceptron}
In this section, we provide an algorithm for the $\ell_2$ costs setting. At a high level, there are two main ideas to modify and generalize the Perceptron algorithm for this setting. The first modification is raising the bar for a point to be classified as positive. Previously, a nonnegative dot product with the current classifier (a threshold of $0$), sufficed for positive classification. However, in the new algorithm, the threshold is a strictly positive value depending on the cost of manipulation. The second modification is using a \emph{surrogate} for the data points when the classifier updates. Interestingly, we only need to use a surrogate for negative points, and in this case the surrogate is a projection of the point in the opposite direction of manipulation, detected by the algorithm.

\begin{figure}
    \centering
        \centering
        \includegraphics[width=0.40\textwidth]{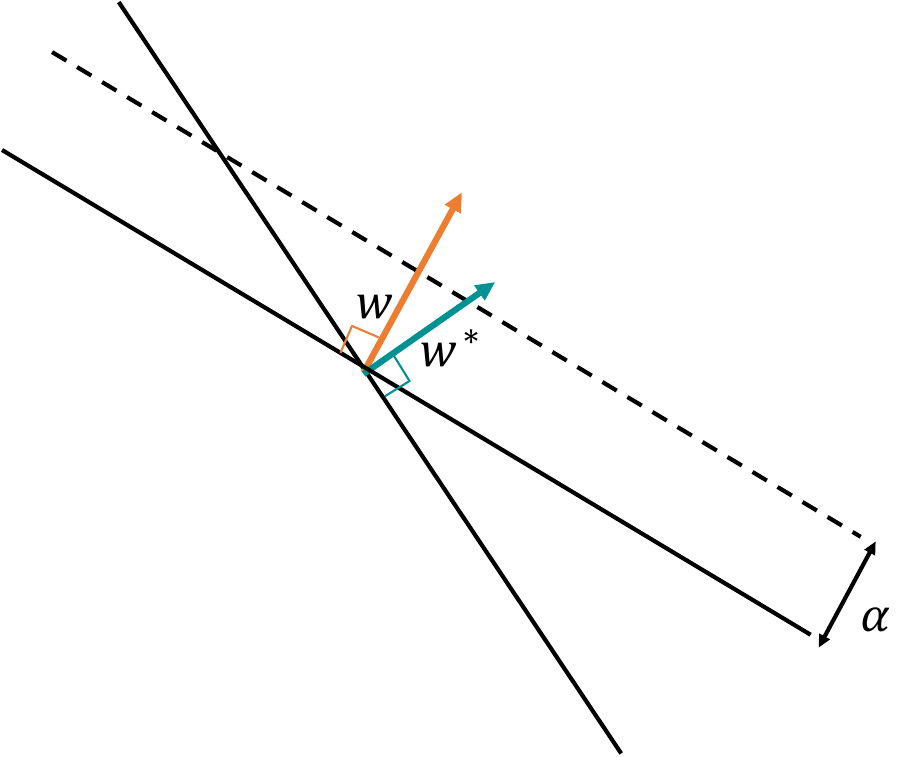}
        \caption{Strategic Perceptron with known manipulation cost. The dashed line represents the manipulation hyperplane discussed in \Cref{obs:manipulation_line}. The margin of width $\alpha$ is the forbidden region, discussed in \Cref{obs:firbidden_region}.}
        \label{fig:Perceptron}
    \hfill
\end{figure}

\paragraph{Overview of \Cref{alg:strategic-Perceptron}.} This algorithm is a generalization of the Perceptron algorithm which we call \emph{strategic Perceptron}. The algorithm starts by predicting all points as positive until it makes a mistake. Note that during this period, individuals do not have incentive to manipulate. From that point on, the algorithm classifies all points with  ${\vec{x}_t^T\vec{w}}/{|\vec{w}|}-\alpha \geq 0$ as positive, and the rest as negative. Whenever the algorithm makes a mistake, the predictor $\vec{w}$ is updated with a surrogate value, $\vec{\tilde{x}}_t$, defined below.

\begin{definition}[$\vec{\tilde{x}}_t$, surrogate data point in $\ell_2$ setting] \label{def:surrogate}
We define surrogate data point, $\vec{\tilde{x}}_t$, as follows.
\begin{align*}
\vec{\tilde{x}}_t=
\begin{cases}
\vec{x}_t-\alpha\frac{\vec{w}}{|\vec{w}|},&\quad\text{if $\vec{x}_t$ is $-$ and }\frac{\vec{x}_t^T\vec{w}}{|\vec{w}|}=\alpha;\\
\vec{x}_t,&\quad\text{if $\vec{x}_t$ is $+$ and }\frac{\vec{x}_t^T\vec{w}}{|\vec{w}|}=\alpha;\\
\vec{x}_t,&\quad\text{if  }\frac{\vec{x}_t^T\vec{w}}{|\vec{w}|} > \alpha \text{ or } \frac{\vec{x}_t^T\vec{w}}{|\vec{w}|} \leq 0.
\end{cases}
\end{align*}
\end{definition}

\begin{algorithm}[tb]
$\vec{w} \leftarrow \vec{0}$\;
\For{$t=1, 2, \cdots$}{
    Given example $\vec{x}_t$:\\
    \If{$|\vec{w}|$ is $0$}{
        predict $+$\;
        \lIf{the prediction was a mistake}{
             $\vec{w} \leftarrow \vec{w}-\vec{x}_t$%
        }
    }
    \Else{
        predict $sgn(\frac{\vec{x}_t^T\vec{w}}{|\vec{w}|}-\alpha)$\;
        \lIf{the prediction was a mistake and $\vec{x}_t$ was $+$}{
        $\vec{w} \leftarrow \vec{w}+\vec{\tilde{x}}_t$%
        }
        \lIf{the prediction was a mistake and $\vec{x}_t$ was $-$}{$\vec{w} \leftarrow \vec{w}-\vec{\tilde{x}}_t$%
        }
    }
}
\caption{Strategic Perceptron for $\ell_2$ costs}
\label{alg:strategic-Perceptron}
\end{algorithm}

\begin{observation}[manipulation hyperplane]\label{obs:manipulation_line}
In \Cref{alg:strategic-Perceptron}, $\vec{x}_t$ is a manipulated example  only if  ${\vec{x}^T\vec{w}}/{|\vec{w}|}=\alpha$. The reason is as follows.
In order to maximize utility, individuals move data points in direction of $\vec{w}$ and move the point the minimum amount to be classified as positive. Therefore, if with true features they are classified as negative, they only need to move to the line with dot product equal to $\alpha$ and moving to any other location contradicts with utility maximizing. {In other words:}
{
  \[\vec{x}_t=
 \begin{cases}
 \vec{z}_t + \Big(\alpha-\frac{\vec{z}_t^T\vec{w}}{|\vec{w}|}\Big)\frac{\vec{w}}{|\vec{w}|}, &\quad\text{if }0 \leq \frac{\vec{z}_t^T\vec{w}}{|\vec{w}|} \leq \alpha;\\
 \vec{z}_t, &\quad\text{otherwise.}
 \end{cases}
 \]
 }
 
\end{observation}

\begin{observation}[forbidden region]\label{obs:firbidden_region}
No observed data point $\vec{x}_t$ will satisfy $0 < {\vec{x}_t^T\vec{w}}/{|\vec{w}|}<\alpha$, and therefore $\vec{\tilde{x}}_t$ does not need to be defined for $0 < {\vec{x}_t^T\vec{w}}/{|\vec{w}|}<\alpha$. The reason is that any such data point must either have manipulated to that position or not.
If it manipulated, the manipulation was not rational since it did not help the data point to get classified as positive. If it did not manipulate, this was not rational either  since the data point has a distance less than $\alpha$ from the classifier.
\end{observation}

We show \Cref{alg:strategic-Perceptron} makes at most $(R+\alpha)^2|\vec{w^*}|^2$ mistakes.
First, we need to prove the following lemmas hold. 
\begin{lemma}
\label{lem:separability}
For any positive data point $\vec{{x}}_t$,  $\vec{\tilde{x}}_t^T\vec{w^*}\geq 1$, and for any negative data point $\vec{{x}}_t$, $\vec{\tilde{x}}_t^T\vec{w^*}\leq -1$. Also,  throughout the execution of \Cref{alg:strategic-Perceptron}, $\vec{w}^T\vec{w^*}\geq 0$.
\end{lemma}
\begin{proof}
The proof uses induction. First, we show after the first update of the algorithm $\vec{w}^T\vec{w^*} > 0$. Second, we show if at the end of step $t-1$, $\vec{w}^T\vec{w^*} \geq 0$, then at step $t$, $\vec{\tilde{x}}_t^T\vec{w^*}\geq 1$ for positive points, and $\vec{\tilde{x}}_t^T\vec{w^*}\leq -1$ for negative points. Finally, we show if $\vec{w}^T\vec{w^*} \geq 0$ at the end of step $t-1$, and $\vec{\tilde{x}}_t^T\vec{w^*}\geq 0$ for positive points, and $\vec{\tilde{x}}_t^T\vec{w^*}\leq 0$ for negative points, then $\vec{w}^T\vec{w^*} \geq 0$ at the end of step $t$. 

The first step is straight-forward. Initially, $\vec{w}=0$. While $\vec{w}=0$, we have $\vec{x}_t^T\vec{w}=0$, and arriving examples get classified positively. The first mistake occurs when a negative example $\vec{x}_t$ arrives, and gets classified as positive. In this case, $\vec{w}$ gets updated to $\vec{w}-\vec{x}_t$. Since $\vec{x}_t^T\vec{w^*}\leq -1$, we conclude $(\vec{w}-\vec{x}_t)\vec{w^*}>0$.

The second step is more involved. By definition of the surrogate values, for any points such that ${\vec{x}_t^T\vec{w}}/{|\vec{w}|} \neq \alpha$,  we have $\vec{\tilde{x}}_t = \vec{x}_t$. By \Cref{obs:manipulation_line}, these points are not manipulated, i.e., $\vec{{x}}_t = \vec{z}_t$. This implies $\vec{\tilde{x}}_t = \vec{z}_t$ and therefore the claim holds. Thus, we only need to argue for the points on the hyperplane ${\vec{x}_t^T\vec{w}}/{|\vec{w}|} = \alpha$. Consider such data points. For the positive data points, we have, $\vec{\tilde{x}}_t = \vec{x}_t = \vec{z}_t + \beta\cdot \vec{w}/|\vec{w}|$, where $0\leq \beta \leq \alpha$. Therefore, $\vec{\tilde{x}}_t^T\vec{w^*} = \vec{z}_t^T\vec{w^*}+\beta\cdot \vec{w}^T\vec{w^*}/|\vec{w}| \geq \vec{z}_t^T\vec{w^*} \geq 1$.
The first inequality holds since by assumption of this step, $\vec{w}^T\vec{w^*}\geq 0$. On the other hand, for the negative data points we have $\vec{\tilde{x}}_t = \vec{x}_t-\alpha \cdot \vec{w}/|\vec{w}|$, where $\vec{x}_t = \vec{z}_t + \beta \cdot \vec{w}/|\vec{w}|$ and $0\leq \beta \leq \alpha$. This implies $\vec{\tilde{x}}_t = \vec{z}_t + (\beta - \alpha) \cdot \vec{w}/|\vec{w}|$. By multiplying with $\vec{w^*}$, we get $\vec{\tilde{x}}_t^T\vec{w^*} = \vec{z}_t^T\vec{w^*}+(\beta-\alpha)\cdot \vec{w}^T\vec{w^*}/|\vec{w}| \leq \vec{z}_t^T\vec{w^*} \leq -1$.

The final step is again straight-forward. Whenever $\vec{w}$ is updated, for positive points, $\vec{w}$ gets updated to $\vec{w}+\vec{\tilde{x}}_t$, where both $\vec{w}$ and $\vec{\tilde{x}}_t$ have nonnegative dot product with $\vec{w^*}$. For negative points, $\vec{w}$ gets updated to $\vec{w}-\vec{\tilde{x}}_t$, where $\vec{w}$ has a nonnegative and $\vec{\tilde{x}}_t$ has a negative dot product with $\vec{w^*}$.
\end{proof}

\begin{lemma}\label{lem:dot_product_x_tilde_w}
When \Cref{alg:strategic-Perceptron} makes a mistake on a positive example $\vec{x}_t$, $\vec{\tilde{x}}_t^T\vec{w}\leq 0$; and when it makes a mistake on a negative example $\vec{x}_t$, $\vec{\tilde{x}}_t^T\vec{w}\geq 0$.
\end{lemma}
\begin{proof}
The algorithm makes a mistake on a positive example only if ${\vec{x}^T\vec{w}}/{|\vec{w}|} < \alpha$. By \Cref{obs:firbidden_region}, for no points, $0 < {\vec{x}^T\vec{w}}/{|\vec{w}|} < \alpha$. Therefore, for any positive example that the algorithm makes a mistake on, $\vec{x}^T\vec{w}\leq 0$. By \Cref{def:surrogate}, $\vec{\tilde{x}}_t = \vec{x}_t$ for all positive examples. Therefore, ${\vec{x}^T\vec{w}}\leq 0$ implies $\vec{\tilde{x}}^T\vec{w}\leq 0$. For negative examples, the algorithm makes a mistake only if ${\vec{x}^T\vec{w}}/{|\vec{w}|} \geq \alpha$. If the inequality is strict, i.e., ${\vec{x}^T\vec{w}}/{|\vec{w}|} > \alpha$, by \Cref{def:surrogate}, $\vec{\tilde{x}}_t=\vec{x}_t$, and therefore $\vec{\tilde{x}}_t^T\vec{w} \geq 0$. If ${\vec{x}^T\vec{w}}/{|\vec{w}|} = \alpha$, again using  \Cref{def:surrogate}, we have $\vec{\tilde{x}}^T\vec{w} = 0$.
\end{proof}

Next, we show the following theorem holds which gives a bound on the number of mistakes. Proof of the following theorem is along the lines of the proof of the classic Perceptron algorithm.
\begin{theorem}
\label{thm:num_updates}
\Cref{alg:strategic-Perceptron} makes at most $(R+\alpha)^2|\vec{w^*}|^2$ mistakes in the strategic setting with known $\ell_2$ costs, when the unmanipulated data points $\vec{z}_t$ satisfy $\vec{z}_t^T\vec{w^*}\geq 1$ for positive examples and $\vec{z}_t^T\vec{w^*}\leq -1$ for negative examples, and $R = \max_t |\vec{z}_t|$.
\end{theorem}

\begin{proof}
We keep track of two quantities, $\vec{w}^T\vec{w^*}$ and $|\vec{w}|^2$. First, we show that each time we make a mistake, $\vec{w}^T\vec{w^*}$ increases by at least $1$. If we make a mistake on a positive example then,
\[(\vec{w}+\vec{\tilde{x}}_t)^T\vec{w}^* = \vec{w}^T\vec{w^*}+\vec{\tilde{x}}_t^T\vec{w}^*\geq \vec{w}^T\vec{w^*} + 1;\]
where the last inequality holds by \Cref{lem:separability}. Similarly, if we make a mistake on a negative example,
\[(\vec{w}-\vec{\tilde{x}}_t)^T\vec{w}^* = \vec{w}^T\vec{w^*}-\vec{\tilde{x}}_t^T\vec{w}^*\geq \vec{w}^T\vec{w^*} + 1.\]
Next, on each mistake we claim that $|\vec{w}|^2$ increases by at most $(R+\alpha)^2$. If we make a mistake on a positive example $\vec{x}_t$, then we have:
\[(\vec{w}+\vec{\tilde{x}}_t)^T(\vec{w}+\vec{\tilde{x}}_t) = |\vec{w}|^2+2\vec{\tilde{x}}_t^T\vec{w}+|\vec{\tilde{x}}_t|^2\leq |\vec{w}|^2+|\vec{\tilde{x}}_t|^2\leq |\vec{w}|^2+ (R+\alpha)^2.\]
To understand the middle inequality note that by \Cref{lem:dot_product_x_tilde_w}, when a mistake is made on a positive example $\vec{x}_t$, $\vec{\tilde{x}}_t^T\vec{w}\leq 0$.
The last inequality comes from $R = \max_t|\vec{z}_t|$ implies $\max_t|\vec{\tilde{x}}_t|
\leq R+\alpha$.

Similarly, if we make a mistake on a negative example $\vec{x}_t$, then we have:
\[(\vec{w}-\vec{\tilde{x}}_t)^T(\vec{w}-\vec{\tilde{x}}_t) = |\vec{w}|^2-2\vec{\tilde{x}}_t^T\vec{w}+|\vec{\tilde{x}}_t|^2\leq |\vec{w}|^2+|\vec{\tilde{x}}_t|^2\leq |\vec{w}|^2+ (R+\alpha)^2.\]

By \Cref{lem:dot_product_x_tilde_w}, when a mistake is made on a negative example $\vec{x}_t$, $\vec{\tilde{x}}_t^T\vec{w}\geq 0$, which implies the middle inequality.

Finally, if the algorithm makes $M$ mistakes, then $\vec{w}^T\vec{w^*}\geq M$ and $|\vec{w}|^2\leq M(R+\alpha)^2$, or equivalently, $|\vec{w}|\leq (R+\alpha)\sqrt{M}$. Using the fact that $\vec{w}^T\vec{w^*}/|\vec{w^*}|\leq |\vec{w}|$, we have
\begin{align*}
M/|\vec{w^*}| &\leq (R+\alpha)\sqrt{M} \implies
\sqrt{M} \leq (R+\alpha)|\vec{w^*}| \implies
M \leq (R+\alpha)^2|\vec{w}^*|^2.
\end{align*}
\end{proof}

%% file: L1.tex
\section{Known Weighted $\ell_1$ Costs}
\label{sec:generalization}
In this section, we provide an algorithm for the weighted $\ell_1$ costs setting. Unlike the $\ell_2$ case, the modifications to the classical Perceptron algorithm in \Cref{alg:strategic-Perceptron} do not suffice; and our algorithm for this setting is more involved. Here is the key difference: In the $\ell_2$ costs setting, the individuals always manipulate in direction of the current classifier $\vec{w}$. However, in the weighted $\ell_1$ setting this is no longer the case. This brings up two challenges to our approach. First, there may be multiple utility maximizing manipulation directions. Second, the manipulation direction may have a negative dot product with $\vec{w^*}$. We overcome these two challenges, and provide an algorithm for this setting.

As a reminder, in the weighted $\ell_1$ costs setting, there are coordinate unit vectors $\{\vec{e}_1,\cdots,\vec{e}_d\}$ with cost of manipulation $1/\alpha_i$ along $\vec{e}_i$. We need to make one further assumption for this setting. We assume for all $1\leq i \leq d$,  $\vec{e}_i^T\vec{w^*} \geq 0$. In other words, we assume that each feature is defined so that larger is better.  This is natural for settings such as hiring, admissions, loan applications, etc.

\paragraph{Overview of \Cref{alg:strategic-Perceptron-L1}.} The algorithm starts by predicting all points as positive until it makes a mistake. Note that during this period, individuals do not have incentive to manipulate.
From that point on, the algorithm classifies all points $\vec{x}_t$ such that ${\vec{x}_t^T\vec{w}}/{|\vec{w}|}-\alpha_i{\vec{w}^T\vec{e}_i}/{|\vec{w}|} \geq 0$ as positive, and the rest as negative; where $\vec{e}_i$ is the manipulation direction which will be defined later. Similar to \Cref{alg:strategic-Perceptron}, whenever the algorithm makes a mistake, the predictor $\vec{w}$ is updated with a surrogate value, $\vec{\tilde{x}}_t$,  in \Cref{def:surrogate_L1}.

Compared to \Cref{alg:strategic-Perceptron}, we have two further steps. As discussed above, the first challenge with weighted $\ell_1$ costs is that with an arbitrary $\vec{w}$, there may be multiple utility maximizing manipulation directions, and we may not be able to distinguish along which vector individuals manipulated. 
Since in the weighted $\ell_1$ costs setting, the cost of manipulation can be written as a convex combination of costs in coordinate vectors, there always exists a coordinate vector, $\vec{e}_i$, such that manipulating along that is utility maximizing. Consider all the coordinate vectors like $\vec{e}_j$ that are utility maximizing, i.e.,  have the highest $\alpha_j\cdot{ \vec{w}^T\vec{e}_j}/{|\vec{w}|}$.
To make the manipulation direction unique, we add a \emph{tie-breaking step} to the algorithm. This step adds a small multiple $\eta>0$, of an arbitrary  utility maximization coordinate vector $\vec{e}_i$, to $\vec{w}$ to break the tie. Note that any positive value of $\eta$ breaks the tie. We set this value in our analysis purposes in \Cref{thm:num-updates-mutiple-direction-manipulation} in a way to make sure the number of mistakes our algorithm makes does not increase much. 

We need to add another step to address the second challenge: With an arbitrary $\vec{w}$ the direction that the individuals manipulate along may not have a positive dot product with $\vec{w^*}$, i.e., the individuals may choose to move along one of the vectors $\{-\vec{e}_1,\cdots,-\vec{e}_d\}$.
In order to incentivize individuals to only manipulate along $\{\vec{e}_1,\cdots,\vec{e}_d\}$, and not  $\{-\vec{e}_1,\cdots,-\vec{e}_d\}$, we do the following \emph{correction step} after each update. If $\vec{e}_j^T\vec{w}<0$ for any $\vec{e}_j\in\{\vec{e}_1,\cdots,\vec{e}_d\}$, we set the $j^{th}$ coordinate of $\vec{w}$ to $0$ by adding the smallest multiple of $\vec{e}_j$, denoted by $\mu_j$, to $\vec{w}$ to make $\vec{e}_j^T\vec{w}$ nonnegative. Therefore, $\mu_j=0$ if $\vec{e}_j^T\vec{w} \geq 0$, and $\mu_j=-\vec{e}_j^T\vec{w}$, otherwise; implying $\forall j$  $\mu_j \geq 0$.

\begin{algorithm}[tb]
\SetNoFillComment
$\vec{w} \leftarrow \vec{0}$\;
\For{$t=1, 2, \cdots$}{
    Given example $\vec{x}_t$:\\
    \If{$|\vec{w}|$ is $0$}{
        predict $+$\;
        \If{the prediction was a mistake}{
            $\vec{w}\leftarrow\vec{w}-\vec{x}_t$\;
            \tcc{Correction Step}
            \For{$j=1, 2, \cdots, d$}{
                $\vec{w}\leftarrow\vec{w} + \mu_j\vec{e}_j$, where $\mu_j = \max(0, -\vec{e}_j^T\vec{w})$\;
            }
            \tcc{Tie-breaking Step}
            $i \leftarrow \argmax_j  \alpha_j\cdot \frac{\vec{w}^T\vec{e}_j}{|\vec{w}|}$\;
            $\vec{w} \leftarrow \vec{w} + \eta \vec{e}_i$\;
        }
    }
    \Else{
        predict $sgn(\frac{\vec{x}_t^T\vec{w}}{|\vec{w}|}-\alpha_i\cdot\frac{\vec{w}^T\vec{e}_i}{|\vec{w}|})$\;
        \lIf{the prediction was a mistake and $\vec{x}_t$ was $+$}{
        $\vec{w} \leftarrow \vec{w}+\vec{\tilde{x}}_t$%
        }
        \lIf{the prediction was a mistake and $\vec{x}_t$ was $-$}{$\vec{w} \leftarrow \vec{w}-\vec{\tilde{x}}_t$%
        }
        \tcc{Correction Step}
        \For{$j=1, 2, \cdots, d$}{
            $\vec{w}\leftarrow \vec{w} + \mu_j\vec{e}_j$ where $\mu_j = \max(0, -\vec{e}_j^T\vec{w})$\;
        }
        \tcc{Tie-breaking Step}
        $i \leftarrow \argmax_j  \alpha_j\cdot \frac{\vec{w}^T\vec{e}_j}{|\vec{w}|}$\;
        $\vec{w} \leftarrow \vec{w} + \eta \vec{e}_i$\;
    }
}
\caption{Strategic Perceptron for weighted $\ell_1$ costs}
\label{alg:strategic-Perceptron-L1}
\end{algorithm}

With the unique manipulation direction, similar to the $\ell_2$ costs setting, we are now able to choose a surrogate value along the manipulation direction.

\begin{definition}[$\vec{\tilde{x}}_t$, surrogate data point in weighted $\ell_1$ setting]\label{def:surrogate_L1}
Let $\vec{e}_i$ be the \emph{unique} utility maximizing coordinate vector, 
i.e., $i = \argmax_j  \alpha_j {\vec{w}^T\vec{e}_j}/{|\vec{w}|}$. We define surrogate data point, $\vec{\tilde{x}}_t$, as follows.
\[\vec{\tilde{x}}_t=
\begin{cases}
\vec{x}_t-\vec{e}_i\cdot\alpha_i,&\quad\text{if $\vec{x}_t$ is $-$ and  }\frac{\vec{x}_t^T\vec{w}}{|\vec{w}|}=\alpha_i\cdot\frac{\vec{w}^T\vec{e}_i}{|\vec{w}|};\\
\vec{x}_t, &\quad\text{otherwise.}
\end{cases}
\]
%\saba{Where: (???)
% \[\vec{x}_t=
% \begin{cases}
% \vec{z}_t + \Big(\alpha_i-\frac{\vec{z}_t^T\vec{w}}{\vec{w}^T\vec{e}_i}\Big)\vec{e}_i &\quad\text{if }0 \leq \frac{\vec{z}_t^T\vec{w}}{|\vec{w}|} \leq \alpha_i\cdot\frac{\vec{w}^T\vec{e}_i}{|\vec{w}|};\\
% \vec{z}_t, &\quad\text{otherwise.}
% \end{cases}
% \]
%}
\end{definition}

\begin{lemma}\label{lem:mu_upperbound}
$\mu_j \leq R+\alpha_j.$
\end{lemma}
\begin{proof}
We can show at the end of each round, $\vec{e}_j^T\vec{w} \geq 0$. Initially, $\vec{w}=0$, therefore $\vec{e}_j^T\vec{w} = 0$.
Suppose at the end of round $t-1$, $\vec{e}_j^T\vec{w} \geq 0$. Assume in round $t$, $\vec{w}$ gets updated by adding or subtracting $\vec{\tilde{x}}_t$ or $\vec{x}_t$. By assumption, the $j^{th}$ coordinate of $\vec{x}_t$ is in $[-R,R]$, and therefore the $j^{th}$ coordinate of $\vec{\tilde{x}}_t$ is in $[-R-\alpha_j, R+\alpha_j]$. Taken together, $\mu_j \leq R+\alpha_j$. Note that by adding $\eta\vec{e}_i$ to $\vec{w}$, $\vec{e}_j^T\vec{w}$ remains nonnegative.
\end{proof}

The following theorem upper bounds the number of mistakes made by \Cref{alg:strategic-Perceptron-L1}.

\begin{theorem}
\label{thm:num-updates-mutiple-direction-manipulation}
Consider a sequence of examples before manipulation $\vec{z}_1, \vec{z}_2,\cdots$, which are observed as $\vec{x}_1, \vec{x}_2,\cdots$.
Consider vector $\vec{w^*}$ such that $\vec{z}_t^T\vec{w^*}\geq 1$ for positive examples, and $\vec{z}_t^T\vec{w^*}\leq -1$ for negative examples. \Cref{alg:strategic-Perceptron-L1} makes at most $(1+(d+1)(R+\alpha)^2)|\vec{w^*}|^2$ mistakes, where $R = \max_t |\vec{z}_t|$, and $\alpha = \max\{\alpha_1,\cdots,\alpha_d\}$.

\end{theorem}
\begin{proof}
Similar to the proof of \Cref{thm:num_updates}, we keep track of two quantities $\vec{w}^T\vec{w^*}$ and $|\vec{w}|^2$. First, we show each time a mistake is made, $\vec{w}^T\vec{w^*}$ increases by at least $1$. Then we find an upper bound on the increase of $|\vec{w}|^2$.

Starting from the current $\vec{w}$, the algorithm follows three steps to update: addition/subtraction of $\vec{\tilde{x}}_t$, the correction step, and the tie-breaking step. As in the algorithm $\vec{e}_i$ is the manipulation direction. 

If the algorithm makes a mistake on a positive example the new value of $\vec{w}$ is $\vec{w}+\vec{\tilde{x}}_t+\eta\vec{e}_i+\sum_j\mu_j\vec{e}_j$. Therefore,
\begin{align*}
\left(\vec{w}+\vec{\tilde{x}_t}+\eta\vec{e}_i+\sum_j\mu_j\vec{e}_j\right)^T\vec{w}^* 
= \vec{w}^T\vec{w^*}+ \vec{\tilde{x}}_t^T\vec{w^*}+\eta\vec{e}_i^T\vec{w^*}+\sum_j\mu_j\vec{e}_j^T\vec{w^*}
\geq \vec{w}^T\vec{w^*} + 1;    
\end{align*}
where the inequality holds because first using the ideas from \Cref{lem:separability}, $\vec{\tilde{x}}_t^T\vec{w^*} \geq 1$ for the positive examples the algorithm makes a mistake on and $\vec{\tilde{x}}_t^T\vec{w^*} \leq -1$ for the negative examples the algorithm makes a mistake on, and second, for all $j$, $\vec{e}_j^T\vec{w^*} \geq 0$ by assumption, and $\mu_j\geq 0$.

Similarly, If the algorithm makes a mistake on a negative example, we have:
\begin{align*}
\left(\vec{w}-\vec{\tilde{x}}_t+\eta\vec{e}_i+\sum_j\mu_j\vec{e}_j\right)^T\vec{w^*} 
= \vec{w}^T\vec{w^*}- \vec{\tilde{x}}_t^T\vec{w^*}+\eta\vec{e}_i^T\vec{w^*}+\sum_j\mu_j\vec{e}_j^T\vec{w^*}
\geq \vec{w}^T\vec{w^*} + 1.    
\end{align*}

Next, on each mistake we claim $|\vec{w}|^2$ increases by at most $(d+1)(R+\alpha)^2 + 1$. If the algorithm makes a mistake on a positive example, we have:
\begin{align*}
&\left|\vec{w}+\vec{\tilde{x}}_t+\eta\vec{e}_i+\sum_j\mu_j\vec{e}_j\right|^2\\
&= \left|\vec{w}+\vec{\tilde{x}}_t+\eta\vec{e}_i\right|^2+\left|\sum_j\mu_j\vec{e}_j\right|^2 + 2\left(\sum_j\mu_j\vec{e}_j\right)^T(\vec{w}+\vec{\tilde{x}}_t+\eta\vec{e}_i)\\
&=|\vec{w}+\vec{\tilde{x}}_t+\eta\vec{e}_i|^2+\sum_j|\mu_j\vec{e}_j|^2 + 2\sum_j\mu_j\vec{e}_j^T(\vec{w}+\vec{\tilde{x}_t}+\eta\vec{e}_i)\\
&\leq |\vec{w}+\vec{\tilde{x}_t}+\eta\vec{e}_i|^2+\sum_j|\mu_j\vec{e}_j|^2 + 2\sum_j\eta\mu_j\vec{e}_j^T\vec{e}_i\\
&=|\vec{w}+\vec{\tilde{x}_t}+\eta\vec{e}_i|^2+\sum_j|\mu_j|^2 + 2\eta\mu_i\\
&=|\vec{w}|^2+|\vec{\tilde{x}}_t|^2+|\eta\vec{e}_i|^2+2\vec{w}^T\vec{\tilde{x}}_t+2\eta\vec{w}^T\vec{e}_i+2\eta\vec{\tilde{x}}_t^T\vec{e}_i+ \sum_j|\mu_j|^2 + 2\eta\mu_i\\
&\leq |\vec{w}|^2+(R+\alpha)^2 + \eta^2+ 0 + 2\eta |\vec{w}| + 2\eta(R+\alpha) + d(R+\alpha)^2 + 2\eta(R+\alpha)\\
&\leq |\vec{w}|^2 + (d+1)(R+\alpha)^2 + \eta^2+ \eta (2|\vec{w}|+4(R+\alpha))\\
&\leq |\vec{w}|^2 + (d+1)(R+\alpha)^2 + 1/4+1/2\\ &\leq |\vec{w}|^2 + (d+1)(R+\alpha)^2 + 1;
\end{align*}
where the first equality is the result of expansion. The second uses $\vec{e}_j^T\vec{e}_k=0$ for $j\neq k$. The inequality in the third row uses $\mu_j=0$ when $\vec{e}_j^T(\vec{w}+\vec{\tilde{x}}_t)\geq 0$, and $\mu_j > 0$ when
$\vec{e}_j^T(\vec{w}+\vec{\tilde{x}}_t) < 0$, implying $\mu\vec{e}_j^T(\vec{w}+\vec{\tilde{x}}_t)\leq 0$. The fourth row uses
$\vec{e}_j^T\vec{e}_k=0$ for $k\neq j$ and $\vec{e_j}^T\vec{e_j}=1$. The fifth row is the result of expansion. The sixth row substitutes each term with an upper bound using $|\vec{\tilde{x}}_t| \leq R+\alpha$ and $\vec{w}^T\vec{\tilde{x}}_t \leq 0$, similar to the arguments from \Cref{lem:dot_product_x_tilde_w}, and $\mu_j \leq R+\alpha$, by \Cref{lem:mu_upperbound}.  
The eighth row results by setting  $\eta = \frac{1}{4|\vec{w}|+8(R+\alpha)+2}$.
The last row sums up and upper bounds similar terms.

Similarly, if the algorithm makes a mistake on a negative example, we have:
\begin{align*}
\left|\vec{w}-\vec{\tilde{x}}_t+\eta\vec{e}_i+\sum_j\mu_j\vec{e}_j\right|^2
&\leq |\vec{w}|^2 + (d+1)(R+\alpha)^2 + 1.
\end{align*}

Therefore, after each mistake, $|\vec{w}|^2$ increases by at most $(d+1)(R+\alpha)^2 + 1$. The rest of the proof is similar to the proof of \Cref{thm:num_updates}, concluding that the total number of mistakes is at most $((d+1)(R+\alpha)^2+1)|\vec{w^*}|^2$.
\end{proof}

%% file: unknown_cost.tex
\section{Unknown Costs}
\label{sec:strategic-Perceptron-unknown-alpha}

\begin{figure}[b]
    \centering
    \includegraphics[width=0.40\textwidth]{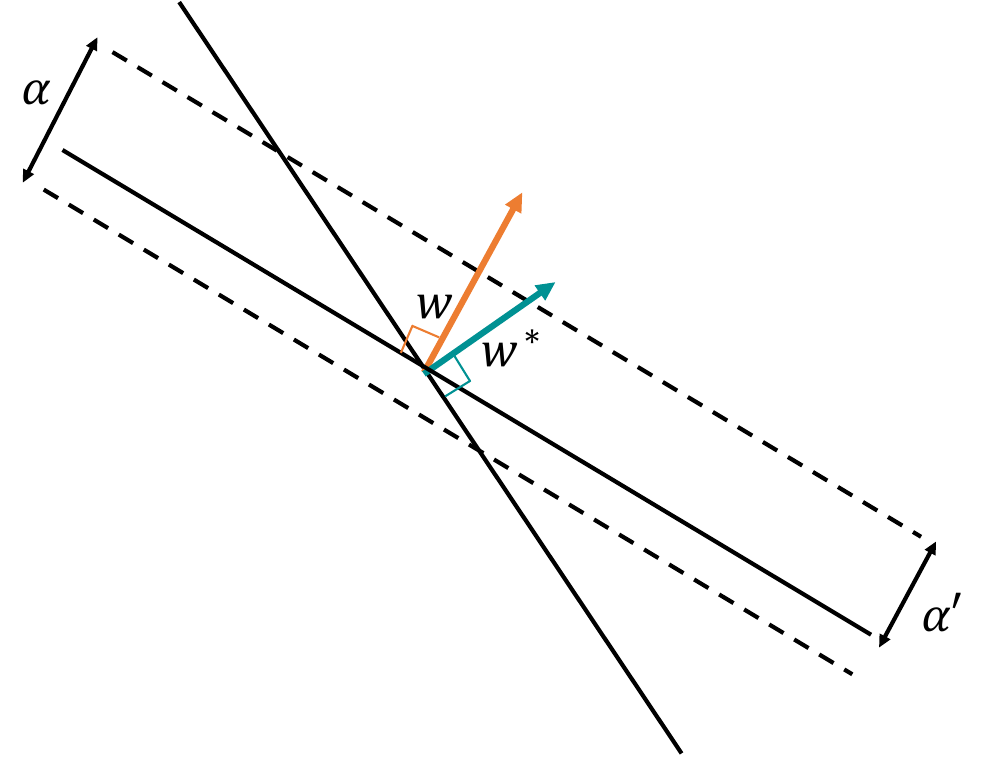}
    \caption{Strategic Perceptron with unknown manipulation cost, when $\alpha \geq \alpha'$. The top dashed line represents the manipulation hyperplane. The margin between the two dashed lines represents the forbidden region.}
    \label{fig:Perceptron-unknown-alpha}
\end{figure}
The main result of this section is generalizing our algorithms to the unknown costs setting. The generalization holds for $\ell_2$ costs . However, it does not extend fully to weighted $\ell_1$ costs  and only works for a specific case. The algorithm for unknown $\ell_2$ costs is presented in \Cref{sec:unknown-L2}. The case of unknown $\ell_1$ costs is studied in \Cref{sec:unknown-L1}.

\subsection{$\ell_2$ Costs}\label{sec:unknown-L2}

In this section, we provide an algorithm that makes at most a bounded number of mistakes when the manipulation cost, $1/\alpha$, is unknown. \Cref{alg:strategic-Perceptron} is used as a subroutine to evaluate our estimate of $\alpha$. First, we show \Cref{alg:strategic-Perceptron} works efficiently if the estimated value, $\alpha'$, is in proximity of the real value (when $\alpha'$ is in the interval of length $\gamma/2$ below $\alpha$). Using this idea we can run a linear search for $\alpha$ with step size $\gamma/2$. However, we show we can do better than a linear search. The key ingredient that lets us outperform the linear search is the ability to distinguish whether the estimate is below or above the real value. Using this idea we run a binary search to find a proper estimate and come up with an efficient algorithm. 

For convenience, we will present the algorithm assuming $\gamma$ is known.  At the end we show how to remove this assumption. Below, we explain these steps more formally.

\subsubsection*{Case 1: $0 \leq \alpha-\alpha' \leq \gamma/2$}
First, we consider the case of $0 \leq \alpha-\alpha' \leq \gamma/2$. Suppose \Cref{alg:strategic-Perceptron} takes $\alpha'$ instead of $\alpha$ as input. Also, suppose $\vec{\tilde{x}}_t$ is defined with respect to $\alpha'$ instead of $\alpha$. In \Cref{pr:num_updates_unknown_alpha}, we show if $0\leq \alpha-\alpha' \leq \gamma/2$, \Cref{alg:strategic-Perceptron} with these modifications, makes at most $4(R+\alpha'+\gamma/2)^2|\vec{w^*}|^2$ mistakes. 
We need the following two lemmas for proving the proposition. Proofs of \Cref{lem:separability-unknown-alpha} \Cref{lem:dot_product_x_tilde_w_unknown} are along the lines of proofs of \Cref{lem:separability,lem:dot_product_x_tilde_w} respectively.
\begin{lemma}
\label{lem:separability-unknown-alpha}
Consider data points $\vec{\tilde{x}}_t$ as defined in \Cref{def:surrogate} w.r.t.\ $\alpha'$ such that $0 \leq \alpha-\alpha' \leq \gamma/2$. These data points are $1/2$-separable; i.e., for positive data points, $\vec{\tilde{x}}_t^T\vec{w^*}\geq 1/2$; and for negative data points, $\vec{\tilde{x}}_t^T\vec{w^*}\leq -1/2$. Also, throughout the execution of \Cref{alg:strategic-Perceptron} with $\alpha'$, $\vec{w}^T\vec{w^*}\geq 0$.
\end{lemma}
\begin{proof}
The proof uses the same three steps as \Cref{lem:separability}. The first and the third steps are identical to \Cref{lem:separability}. Here, we argue for the second step, i.e., if at the end of step $t-1$, $\vec{w}^T\vec{w^*} \geq 0$, then at step $t$, $\vec{\tilde{x}}_t^T\vec{w^*}\geq 1/2$ for positive points, and $\vec{\tilde{x}}_t^T\vec{w^*}\leq -1/2$ for negative points.

When \Cref{alg:strategic-Perceptron} is run with $\alpha'$, by \Cref{def:surrogate}, for any points such that ${\vec{x}_t^T\vec{w}}/{|\vec{w}|} \neq \alpha'$,  we have $\vec{\tilde{x}}_t = \vec{x}_t$. By \Cref{obs:manipulation_line}, these points are not manipulated, i.e., $\vec{{x}}_t = \vec{z}_t$. This implies $\vec{\tilde{x}}_t = \vec{z}_t$ which implies the claim for these points. Thus, we only need to argue for the data points such that ${\vec{x}_t^T\vec{w}}/{|\vec{w}|} = \alpha'$. Consider such data points. For the positive data points, we have, $\vec{\tilde{x}}_t = \vec{x}_t = \vec{z}_t + \beta\cdot \vec{w}/|\vec{w}|$, where $0\leq \beta \leq \alpha$. Therefore, $\vec{\tilde{x}}_t^T\vec{w^*} = \vec{z}_t^T\vec{w^*}+\beta\cdot \vec{w}^T\vec{w^*}/|\vec{w}| \geq \vec{z}_t^T\vec{w^*} \geq 1$.
The first inequality holds because by  the assumption of this step, $\vec{w}^T\vec{w^*}\geq 0$. On the other hand, for the negative data points we have $\vec{\tilde{x}}_t = \vec{x}_t-\alpha' \cdot \vec{w}/|\vec{w}|$, where $\vec{x}_t = \vec{z}_t + \beta \cdot \vec{w}/|\vec{w}|$ and $0\leq \beta \leq \alpha$. This implies $\vec{\tilde{x}}_t = \vec{z}_t + (\beta - \alpha') \cdot \vec{w}/|\vec{w}|$. By multiplying with $\vec{w^*}$, we get $\vec{\tilde{x}}_t^T\vec{w^*} = \vec{z}_t^T\vec{w^*}+(\beta-\alpha')\cdot \vec{w}^T\vec{w^*}/|\vec{w}| \leq \vec{z}_t^T\vec{w^*}+(\alpha-\alpha')\cdot \vec{w}^T\vec{w^*}/|\vec{w}|$. Using $0 \leq \alpha-\alpha'\leq \gamma/2$ and $\gamma=1/|\vec{w^*}|$, we have $\vec{\tilde{x}}_t^T\vec{w^*} \leq \vec{z}_t^T\vec{w^*}+ \vec{w}^T\vec{w^*}/(2|\vec{w^*}||\vec{w}|) \leq \vec{z}_t^T\vec{w^*} + 1/2 \leq -1/2$.
\end{proof}

\begin{lemma}\label{lem:dot_product_x_tilde_w_unknown}
Suppose \Cref{alg:strategic-Perceptron} is run with $\alpha'$ such that $0\leq \alpha-\alpha' \leq \gamma/2$.
When the algorithm makes a mistake on a positive example $\vec{x}_t$, $\vec{\tilde{x}}_t^T\vec{w}\leq 0$; and when it makes a mistake on a negative example $\vec{x}_t$, $\vec{\tilde{x}}_t^T\vec{w}\geq 0$.
\end{lemma}
\begin{proof}
First, we consider the positive points. The algorithm makes a mistake on a positive example only if ${\vec{x}^T\vec{w}}/{|\vec{w}|} < \alpha'$. Similar to \Cref{obs:firbidden_region}, in this case there is a margin without any observed data points. However, as illustrated in \Cref{fig:Perceptron-unknown-alpha}, this margin is located differently; such that for no points, $\alpha'-\alpha < {\vec{x}^T\vec{w}}/{|\vec{w}|} < \alpha'$. Thus, for any positive example that the algorithm makes a mistake on, $\vec{x}^T\vec{w}\leq 0$. By \Cref{def:surrogate}, $\vec{\tilde{x}}_t = \vec{x}_t$ for all positive examples. Therefore, ${\vec{x}^T\vec{w}}\leq 0$ implies $\vec{\tilde{x}}^T\vec{w}\leq 0$. Second, we consider negative points. For negative examples, the algorithm makes a mistake only if ${\vec{x}^T\vec{w}}/{|\vec{w}|} \geq \alpha'$. If the inequality is strict, i.e., ${\vec{x}^T\vec{w}}/{|\vec{w}|} > \alpha'$, by \Cref{def:surrogate}, $\vec{\tilde{x}}_t=\vec{x}_t$, and therefore $\vec{\tilde{x}}_t^T\vec{w} \geq 0$. If ${\vec{x}^T\vec{w}}/{|\vec{w}|} = \alpha'$, again using  \Cref{def:surrogate}, we have $\vec{\tilde{x}}^T\vec{w} = 0$.
\end{proof}

\begin{proposition}
\label{pr:num_updates_unknown_alpha}
When $0\leq \alpha-\alpha' \leq \gamma/2$, \Cref{alg:strategic-Perceptron}
makes at most $4(R+\alpha'+\gamma/2)^2|\vec{w^*}|^2$ mistakes.
\end{proposition}
\begin{proof}
Using \Cref{lem:separability-unknown-alpha,lem:dot_product_x_tilde_w_unknown}, the rest of the proof is similar to \Cref{thm:num_updates} and is deferred to the Appendix.
\end{proof}
\subsubsection*{Case 2: $\alpha < \alpha'$}
Suppose $\alpha'$ is larger than $\alpha$. By \Cref{obs:firbidden_region}, when \Cref{alg:strategic-Perceptron} is run with the real value of $\alpha$, no data point is observed by algorithm in the margin $0 < \vec{x}_t^T \vec{w}/|\vec{w}|<\alpha$. However, when the estimate is larger, since we overestimate by how far individuals can manipulate, \Cref{obs:firbidden_region} no longer holds. Therefore, if the algorithm observes a point in the margin $0 < \vec{x}_t^T \vec{w}/|\vec{w}|<\alpha'$, we realize that the estimate is large, and we need to refine it. On the other hand, while we have not observed any such points, the algorithm makes at most $(R+\alpha')^2|\vec{w^*}|^2$ mistakes. This statement is summarized and proved below.

\begin{proposition}
\label{pr:num_updates_unknown_alpha_large_guess}
Suppose \Cref{alg:strategic-Perceptron} is run with $\alpha'$, such that $\alpha' > \alpha$, and is halted if for a data-point $\vec{x}_t$, $0 < {\vec{x}_t^T\vec{w}}/{|\vec{w}|} < \alpha'$. This modified algorithm
makes at most $(R+\alpha')^2|\vec{w^*}|^2+1$ mistakes.
\end{proposition}
\begin{proof}
Similar to the proof of \Cref{thm:num_updates}, the maximum number of mistakes \Cref{alg:strategic-Perceptron} with estimated manipulation cost $1/\alpha'$ makes on observed data points $\vec{x}_t$ where ${\vec{x}_t^T\vec{w}}/{|\vec{w}|} \leq 0$ or ${\vec{x}_t^T\vec{w}}/{|\vec{w}|}\geq\alpha'$ is at most $(R+\alpha')^2|\vec{w^*}|^2$. If a data point $\vec{x}_t$ is observed such that $0 < {\vec{x}_t^T\vec{w}}/{|\vec{w}|}<\alpha'$, it implies $\alpha'>\alpha$ and the algorithm halts, and at most one more mistake is made on this data point. Therefore, the total number of mistakes is at most $(R+\alpha')^2|\vec{w^*}|^2+1$.
\end{proof}

\subsubsection*{Case 3: $\alpha' < \alpha - \gamma/2$}\label{sec:last_case}
We infer from \Cref{pr:num_updates_unknown_alpha,pr:num_updates_unknown_alpha_large_guess} that if the number of mistakes is greater than $\max\{4(R+\alpha'+\gamma/2)^2|\vec{w^*}|^2,(R+\alpha')^2|\vec{w^*}|^2+1\} = 4(R+\alpha'+\gamma/2)^2|\vec{w^*}|^2$ then $\alpha' < \alpha - \gamma/2$. Note that the equality holds since the number of mistakes is an integer.

\subsubsection*{Putting Everything Together}

After discussing the three cases, we are now ready to explain \Cref{alg:strategic-Perceptron-unknown-alpha}. This algorithm, uses a binary search scheme to find a predictor in a bounded number of mistakes. The algorithm starts with $\alpha' = 0$. For each fixed $\alpha'$ we consider $4(R+\alpha'+\gamma/2)^2|\vec{w^*}|^2$ as the maximum number of allowed mistakes. Whenever we exceed this bound using the discussion in \Cref{sec:last_case} we learn that $\alpha'$ is too small. Also whenever we see a data point $\vec{x}_t$ such that $0\leq \vec{x}_t^T \vec{w}/|\vec{w}|<\alpha$ as explained above we learn that $\alpha'$ is too large. Distinguishing between the cases where $\alpha'$ is too large or too small allows us to refine the upper bound and lower bound on $\alpha'$ until $0 \leq \alpha-\alpha' \leq \gamma/2$. The following theorem shows that the total number of mistakes is bounded during the whole process.

\begin{algorithm}[tb]
\SetNoFillComment
\caption{Strategic Perceptron with unknown manipulation cost}
\label{alg:strategic-Perceptron-unknown-alpha}
   $\alpha''\leftarrow 0, \alpha' \leftarrow 0$\;
    \While{examples are arriving}{
    Run \Cref{alg:strategic-Perceptron} with estimate $\alpha'$ on the sequence of arriving examples, halt if $\#mistakes  > 4(R+\alpha'+\gamma/2)^2|\vec{w^*}|^2$ or if for an example $\vec{x}_t$, $0 < \frac{\vec{x}_t^T\vec{w}}{|\vec{w}|}<\alpha'$\;
   \If{$\#mistakes > 4(R+\alpha'+\gamma/2)^2|\vec{w^*}|^2$}{
   \tcc{guessed value $\alpha'$ is small.}
   $\alpha''\leftarrow\alpha'$\;
   $\alpha' \leftarrow \min\{\max\{2\alpha',\gamma/2\}, R\}$\;
   continue\;
   }
   \ElseIf{for an example $\vec{x}_t$, $0 < \frac{\vec{x}_t^T\vec{w}}{|\vec{w}|}<\alpha'$}{
   \tcc{guessed value $\alpha'$ is large.}
    $\alpha' \leftarrow (\alpha''+\alpha')/2$\;
    continue\;
    }
   }
\end{algorithm}

\begin{theorem}
\label{thm:l2-unknown-cost}
\Cref{alg:strategic-Perceptron-unknown-alpha} makes at most $\mathcal{O}(R^2|\vec{w^*}|^2\log(R|\vec{w^*}|))$ mistakes.
\end{theorem}
\begin{proof}
In \Cref{alg:strategic-Perceptron-unknown-alpha}, the candidates for $\alpha$ are $\gamma/2$ apart and the number of them is $2R|\vec{w^*}|$. Since we are doing a binary search on these candidates, the total number of iterations of binary search is at most $\log(2R|\vec{w^*}|)$. \Cref{pr:num_updates_unknown_alpha}, \Cref{pr:num_updates_unknown_alpha_large_guess}, and \Cref{thm:num_updates}, show that in each iteration the total number of mistakes is bounded by $\max\{4(R+\alpha'+\gamma/2)^2|\vec{w^*}|^2,(R+\alpha')^2|\vec{w^*}|^2+1\}$. 
Since we are assuming $\alpha'\leq R$, the total number of mistakes is at most $\mathcal{O}(R^2|\vec{w^*}|^2\cdot\log(R|\vec{w^*}|))$ and the proof is complete.
\end{proof}

\subsubsection*{Unknown $\gamma$}

In the previous steps we assumed knowledge of $\gamma$. However, this assumption is not necessary and we can remove it in the following way. Starting from a guess of $|\vec{w}^*|=\frac{1}{2R}$ (i.e., a guess of $\gamma = 2R$), repeat the following procedure: 
for each guessed value of $|\vec{w^*}|$, \Cref{alg:strategic-Perceptron-unknown-alpha} is executed and if it makes more than the mistake bound of $\mathcal{O}(R^2|\vec{w^*}|^2\log(R|\vec{w^*}|))$, the guessed value for $|\vec{w^*}|$ is doubled (i.e., the guessed value of $\gamma$ is halved) and the procedure is repeated. 
We show by putting this wrapper around \Cref{alg:strategic-Perceptron-unknown-alpha}, the total number of mistakes remains in the same order of magnitude:

\begin{align*}
\sum_{i=-1}^{\log{2R|\vec{w^*}|}} R^2\left(\frac{|\vec{w^*}|}{2^{i}}\right)^2\log\left(\frac{R|\vec{w^*}|}{2^i}\right) = \mathcal{O}(R^2|\vec{w^*}|^2\log(R|\vec{w^*}|))
\end{align*}

\subsection{Weighted $\ell_1$ Costs}\label{sec:unknown-L1}
As observed in \Cref{sec:generalization}, in order for the strategic Perceptron algorithm to work in the weighted $\ell_1$ costs model, it is necessary to identify in what direction the individuals manipulate. The tie-breaking step in \Cref{alg:strategic-Perceptron-L1}, ensured that the manipulation direction is unique and identifiable. In the unknown costs model,  we need to make a guess for the cost in each direction. Since the guessed values are not accurate, we no longer can use them for a tie-breaking step and determine the manipulation direction. This restrains us from having an efficient algorithm for the general case of $\ell_1$ costs. However, for a special case where manipulation is possible in a single direction (finite cost in direction $\vec{e}_1$ and infinite in the others), the manipulation direction is known and the ideas of \Cref{alg:strategic-Perceptron-unknown-alpha} extend to this case. 

%% file: different_cost.tex
\section{Different Costs}\label{sec:different_costs}

In the previous sections, we assumed all individuals have the same utility function. In this section, we show this assumption is not critical for our result and our algorithms still make a bounded number of mistakes and perform almost as well as long as the utility functions are close enough.

More particularly, suppose in the $\ell_2$ costs setting, at each time $t$, the amount that an individual can move, $\alpha_t$, is upper bounded by $\alpha_{\max}$ and lower bounded by $\alpha_{\min}$ such that $0 \leq \alpha_{\max}-\alpha_{\min}\leq \gamma/2$. Using the ideas presented in \Cref{sec:unknown-L2}, we can show that running \Cref{alg:strategic-Perceptron} with $\alpha_{\min}$ as the input and the surrogate data points $\vec{\tilde{x}}_t$ defined with respect to $\alpha_{\min}$ makes a bounded number of mistakes.

\begin{corollary}\label{cor:different_cost}
Suppose for all $t$, $\alpha_{\min} \leq \alpha_t \leq \alpha_{\min}+\gamma/2$. \Cref{alg:strategic-Perceptron} by using parameter $\alpha_{\min}$ makes at most $4(R+\alpha_{\min}+\gamma/2)|\vec{w^*}|^2$ number of mistakes. 
\end{corollary}

\begin{proof}[Proof Outline]
In \Cref{sec:unknown-L2}, the guessed value of $\alpha$ that is used as the input to \Cref{alg:strategic-Perceptron}, is at most $\gamma/2$ smaller than the real value. Similarly, in this case, $\alpha_{\min}$ is at most $\gamma/2$ smaller than any $\alpha_t$.  With a small difference in the terminology of the proofs of \Cref{lem:separability-unknown-alpha,lem:dot_product_x_tilde_w_unknown}, their statements hold and \Cref{pr:num_updates_unknown_alpha} directly implies this corollary.
\end{proof}

\subsubsection*{Weighted $\ell_1$ Costs} Due to similar reasons explained in \Cref{sec:unknown-L1}, the previous result does not extend to general case of weighted $\ell_1$ costs but extends to the special case where manipulation is possible in a single direction.

%% file: non_zero_bias.tex
\section{Target Classifier Not Crossing the Origin}\label{sec:bias}

In this section, we propose an algorithm for the setting where unmanipulated data points are linearly separable, however not by a linear separator passing through the origin. Ordinarily (in the non-strategic setting), this would be handled by
creating an extra fake coordinate, giving each example a value of 1 in that coordinate, and thereby reducing to the case where the separator crosses the origin, as explained in \Cref{ext:bias}. However, in the strategic setting, this reduction breaks down because the condition that $\vec{w}^T\vec{w^*}\geq 0$ (given in \Cref{lem:separability}) is no longer sufficient to guarantee the quality of $\vec{\tilde{x}}_t$, since agents cannot manipulate in this new coordinate (they are no longer manipulating in the direction of $\vec{w}$). Instead, we present a different reduction here that is robust to strategic behavior.  

\begin{algorithm}[tb]
\SetNoFillComment
$\vec{x^+}, \vec{x^-} \leftarrow \vec{0}, \vec{0}$\;
    \While{examples are arriving}{
        predict $+$\;
        \If{the prediction was a mistake on a current example $\vec{x_t}$}{
            $\vec{x^-}\leftarrow \vec{x}_t$\;
            break\;
        }
    }
    \While{examples are arriving}{
        predict $-$\;
        \If{the prediction was a mistake on a current example $\vec{x_t}$}{
            $\vec{x^+}\leftarrow \vec{x}_t$\;
            break\;
        }
    }
        $\lambda\leftarrow 0$\;
        \While{examples are arriving}{
            $mistakes\leftarrow 0$\;
            \tcc{choose a point $\vec{p}$ on the line segment between $\vec{x^+}$ and $\vec{x^-}$.}
            $\vec{p}\leftarrow (1-\lambda)\vec{x^-}+\lambda\vec{x^+}$\;
            \tcc{set the origin to point $\vec{p}$.}
            Run \Cref{alg:strategic-Perceptron} on the sequence of arriving examples in the new coordinate system, i.e. replace each example $\vec{x}_t$ with $\vec{x_t}-\vec{p}$, and halt if $mistakes > 4(2R+\alpha)^2|\vec{w^*}|^2$\;
            \tcc{$\vec{p}$ is not close enough to $\vec{p^*}$, i.e. $|\vec{p}-\vec{p^*}|>\gamma/2$. Try a different $\vec{p}$.}
            $\lambda\leftarrow \lambda +\gamma/2$\;
            continue\;
        }
\caption{Strategic Perceptron with bias for $\ell_2$ costs}
\label{alg:strategic-Perceptron-bias}
\end{algorithm}

For the case of $\ell_2$ costs, we provide  \Cref{alg:strategic-Perceptron-bias} and show that the number of mistakes it makes is at most $\mathcal{O}(R^3|\vec{w^*}|^3)$.

\paragraph{Overview of \Cref{alg:strategic-Perceptron-bias}.} Assume the unmanipulated data points are separable by a linear separator $\vec{w^*}^T\vec{z}+b = 0$. Suppose we can find an arbitrary point $\vec{p^*}$ such that $\vec{w^*}^T\vec{p^*}+b = 0$. If we set the point $\vec{p^*}$ as the new origin, i.e., replacing each example $\vec{z}_t$ with $\vec{z}_t-\vec{p^*}$, then in the new coordinate system the unmanipulated data points 
are linearly separable by a separator that crosses the new origin, and we can use our previous algorithms. However, we are not able to necessarily find a point $\vec{p^*}$ such that $\vec{w^*}^T\vec{p^*}+b = 0$.
Instead, we show how to find a point $\vec{p}$ that is close enough to $\vec{p^*}$.

Initially, we find an unmanipulated positive example $\vec{x}^+$, and an unmanipulated negative example $\vec{x}^-$ by starting our algorithm in the following way: First, predict positive until the first mistake on a negative example $\vec{x}^-$ is made. Next, predict negative until the next mistake is made on some positive example $\vec{x}^+$. Consider the line segment between $\vec{x}^-$ and $\vec{x}^+$. There exists a point $\vec{p^*}$ on this line segment where $\vec{w^*}^T\vec{p^*}+b = 0$. Consider a series of points on this line segment at distance $\gamma/2$ apart. For one of these points, which we call $\vec{p}$, $|\vec{p}-\vec{p^*}|\leq \gamma/2$. \Cref{lem:non_zero_bias_separability} shows if the origin is set to $\vec{p}$, i.e. each data point $\vec{z}_t$ is replaced with $\vec{z}_t-\vec{p}$, there exists a line passing through the new origin $\vec{p}$ that separates original data points with a margin of $\gamma/2$, meaning that for each unmanipulated negative example $\vec{z}_t$, $(\vec{z}_t-\vec{p})^T\vec{w^*}\geq 1/2$, and for each unmanipulated negative example $\vec{z}_t$, $(\vec{z}_t-\vec{p})^T\vec{w^*}\leq -1/2$. When the origin is set to $\vec{p}$, \Cref{lem:non_zero_bias_num_updates} shows that if \Cref{alg:strategic-Perceptron} is executed on the arrived examples in the new coordinate system, i.e. replacing each observed example $\vec{x}_t$ with $\vec{x_t}-\vec{p}$, the number of mistakes is at most $4(2R+\alpha)^2|\vec{w^*}|^2$.
Putting all together, we propose \Cref{alg:strategic-Perceptron-bias} that is a generalization of \Cref{alg:strategic-Perceptron} for the case that original examples are separable by a linear classifier with non-zero bias. \Cref{thm:bias} shows \Cref{alg:strategic-Perceptron-bias} makes at most $\mathcal{O}(R^3|\vec{w^*}|^3)$ mistakes.

\begin{lemma}
\label{lem:non_zero_bias_separability}
Assume the points $\vec{z}_t$ are separable by a linear separator $\vec{w^*}^T\vec{z}+b = 0$ of margin $\gamma = 1/|\vec{w^*}|$, and let $\vec{p^*}$ be a point such that $\vec{w^*}^T\vec{p^*}+b = 0$.  Then, if 
$|\vec{p^*}-\vec{p}|\leq \gamma/2$, the decision boundary $(\vec{z}-\vec{p})^T\vec{w^*}$ has a margin of separation $\gamma/2$.
\end{lemma}
\begin{proof}
First, $|\vec{w^*}^T\vec{p}-\vec{w^*}^T\vec{p^*}|\leq 1/2$ because $|\vec{w^*}^T\vec{p}-\vec{w^*}^T\vec{p^*}|\leq |\vec{w^*}||\vec{p}-\vec{p^*}|\leq |\vec{w^*}|/2|\vec{w^*}|=1/2$.  So, for a positive data point $\vec{z}_t$, if $(\vec{z}_t-\vec{p^*})^T\vec{w^*}\geq 1$, then $(\vec{z}_t-\vec{p})^T\vec{w^*}\geq 1/2$. 
Similarly, for a negative data point $\vec{z}_t$, if $(\vec{z}_t-\vec{p^*})^T\vec{w^*}\leq -1$ then $(\vec{z}_t-\vec{p})^T\vec{w^*}\leq -1/2$. 
\end{proof}

\begin{lemma}
\label{lem:non_zero_bias_num_updates}
For a fixed guess $\vec{p}$ where $|\vec{p^*}-\vec{p}|\leq \gamma/2$, when the origin is set to $\vec{p}$ (i.e., each example $\vec{x}$ is replaced by $\vec{x}-\vec{p}$), then \Cref{alg:strategic-Perceptron-bias} makes at most $4(2R+\alpha)^2|\vec{w^*}|^2$ mistakes.
\end{lemma}

\begin{proof}
Proof of this lemma is in the same lines as the proof of \Cref{thm:num_updates} with some modifications. First, \Cref{lem:non_zero_bias_separability} shows there exists a separator with margin of separation $\gamma/2$ passing through $\vec{p}$. By following the steps in \Cref{lem:separability}, and using a margin of separation $\gamma/2$ instead of $\gamma$, we can show for any positive data point $\vec{{x}}_t$,  $\vec{\tilde{x}}_t^T\vec{w^*}\geq 1/2$, and for any negative data point $\vec{{x}}_t$, $\vec{\tilde{x}}_t^T\vec{w^*}\leq -1/2$. Next, since each data point $\vec{x}_t$ is moved to $\vec{x}_t-\vec{p}$, $\max_t|\vec{\tilde{x}}_t|
\leq 2R+\alpha$.
Finally, by applying these bounds and following the steps in the proof of \Cref{thm:num_updates}, the claim is proved.
\end{proof}

\begin{theorem}\label{thm:bias}
\Cref{alg:strategic-Perceptron-bias} makes at most $16R(2R+\alpha)^2|\vec{w^*}|^3$ mistakes in total.
\end{theorem}

\begin{proof}
Since the length of the line segment between $\vec{x^+}$ and $\vec{x^-}$ is at most $2R$, and guesses tried on this line segment are $\gamma/2$ apart, \Cref{alg:strategic-Perceptron-bias} tries at most $4R/\gamma$ guesses for $\vec{p} $ in total.  For each guess, if the number of mistakes is greater than $4(2R+\alpha)^2|\vec{w^*}|^2$, the next guess is tried. By \Cref{lem:non_zero_bias_num_updates}, if for a guess $\vec{p}$, $|\vec{p}-\vec{p^*}|\leq \gamma/2$, on all the examples that will arrive the number of mistakes does not exceed $4(2R+\alpha)^2|\vec{w^*}|^2$. Therefore the claim is proved. 
\end{proof}

\subsubsection*{Weigthed $\ell_1$ Costs}

We show a reduction from this setting to the case where unmanipulated examples are separable by a hyperplane passing through the origin. First, similar to the $\ell_2$ case, an extra fake coordinate with a value of $1$ is added to each example. The bias term $b$ is absorbed into $\vec{w^*}$, i.e. replacing $\vec{w^*}$ with $(\vec{w^*},b)$. Now for the positive examples $\vec{x}_t^T\vec{w^*}\geq 1$, and for the negative examples $\vec{x}_t^T\vec{w^*}\leq -1$.
Since agents cannot manipulate in the direction of the fake coordinate, the cost of manipulation along this direction is set to be infinity. Next we bound the number of mistakes that \Cref{alg:strategic-Perceptron-L1} makes in this setting. Since a fake coordinate is added to each example, $R$ increases by a value of at most $1$. Since the bias term is absorbed into the $\vec{w^*}$, the value of $|\vec{w}|$ increases by at most $b$. As a result, \Cref{alg:strategic-Perceptron-L1} makes at most $(1+(d+1)(R+\alpha+1)^2)(|\vec{w^*}|+b)^2$ number of mistakes.  At the high level, the reason the direct reduction goes through for the weighted $\ell_1$ case but not the $\ell_2$ case is that in the $\ell_1$ case, agents only manipulate in coordinate directions, and we have assumed that $\vec{w^*}$ is non-negative in each coordinate direction.  So, the algorithm is not hurt if in computing $\tilde{\vec{x}}_t$ it overestimates the amount by which the agent has manipulated.  This is the property that breaks down in the $\ell_2$ case.

%% file: conclusions.tex
\section{Conclusions and Open Problems}
\label{sec:conclusion}
In this work, we showed that if agents have the ability to manipulate their features within an $\ell_2$ ball or a weighted $\ell_1$ ball in order to be classified as positive, then the classic Perceptron algorithm may fail to achieve a bounded number of mistakes even when a perfect linear classifier exists.
We then developed new Perceptron-style algorithms that achieve a finite mistake-bound, not much greater than the classic Perceptron bound in the non-strategic case, in both the $\ell_2$ and weighted $\ell_1$ manipulation setting. In the case that the manipulation costs are unknown to the learner---i.e., the radius of the ball in which agents can modify their features (or the per-coordinate radius in the weighted $\ell_1$ case)---we 
provide an algorithm for the $\ell_2$ costs setting and a specific case of the weighted $\ell_1$ costs setting.

Our work suggests several open problems. First, designing an algorithm for the general case of weighted  $\ell_1$ costs when the costs of manipulation along each coordinate is unknown. This is challenging because given an observed data point, the learner doesn't know which direction it may have manipulated from, and this direction will change as the hypothesis classifier changes. 

Second, for the case of inseparable data points, getting a bound in terms of the hinge-loss of the best separator with respect to the original data points $\vec{z}_1,\vec{z}_2,\cdots$. Our ideas in \Cref{sec:strategic-Perceptron} can be extended to get a bound in terms of the hinge-loss of the best separator of surrogate data points $\vec{\tilde{x}}_1, \vec{\tilde{x}}_2,\cdots$. However, the more interesting question of getting a bound in terms of unmanipulated data points remains open. %\Cref{ex:hinge} 
In the following, we show an example where \Cref{alg:strategic-Perceptron} makes an unbounded number of mistakes when data points are not perfectly separable, even though there exists a separator with bounded hinge-loss.

\begin{example}\label{ex:hinge}
 Consider data points $\vec{z}_0=(4, 3)$, $\vec{z}_1=(-1, -7)$, $\vec{z}_2=(3,2)$, $\vec{z}_{3}= (-1,7)$, and $\vec{z}_{4} = (3, -2)$ arriving in order; and then the examples $\vec{z}_1, \vec{z}_2,\vec{z}_3,\vec{z}_4$ repeat forever.  Examples $\vec{z}_2$ and $\vec{z}_4$ have positive labels, and $\vec{z}_0$, $\vec{z}_1$ and $\vec{z}_3$ have negative labels. Suppose that $\alpha=5$.  Note that there exists a vertical linear separator that only makes a mistake on $\vec{z}_0$.  However, as shown below, \Cref{alg:strategic-Perceptron} will make an unbounded number of mistakes.
 
Specifically,
after arrival of $\vec{z}_0$, we have $\vec{w}=(-4, -3)$. Next, $\vec{z}_1$ arrives, and since $\vec{w}^T\vec{z}_1/|\vec{w}| = \alpha$, the algorithm makes a mistake on $\vec{z}_1$ and classifies it as positive. $\vec{z}_1$ doesn't even have to manipulate, i.e. $\vec{z}_1 = \vec{x}_1$.  Surrogate data point $\vec{\tilde{x}}_1=(3,-4)$ is created, and is subtracted from $\vec{w}$ to get $\vec{w}=(-7, 1)$. Example $\vec{z}_2$ arrives and since $\vec{w}^T\vec{z}_2 < 0$, manipulation does not help, therefore, $\vec{x}_2  = \vec{z}_2$. Example $\vec{z}_2$ gets classified as negative mistakenly as shown in \Cref{fig:hing2}. Also, $\vec{\tilde{x}}_2=\vec{x}_2$. Since a mistake is made, $\vec{w}$ gets updated to $\vec{w}+\vec{\tilde{x}}_2 = (-4, 3)$. Next, negative example $\vec{z}_3$ arrives and since $\vec{w}^T\vec{z}_3/|\vec{w}| = \alpha$, it gets misclassified as positive as shown in \Cref{fig:hing3}, and $\vec{x}_3 = \vec{z}_3$. Surrogate data point $\vec{\tilde{x}_{3}} = (3,4)$ is created and $\vec{w}$ is updated to $\vec{w} - \vec{\tilde{x}_{3}} = (-7, -1)$. Next, positive example $\vec{z}_{4}$ arrives, and since $\vec{w}^T\vec{z}_4 < 0$, it does not manipulate and $\vec{z}_4 = \vec{x}_4$. It gets classified as negative mistakenly as shown in \Cref{fig:hing4}. Also, $ \vec{\tilde{x}}_4 = \vec{x}_4$. $\vec{w}$ is updated to $\vec{w}+\vec{\tilde{x}}_4 = (-4, -3)$. If the same four examples arrive over and over again, \Cref{alg:strategic-Perceptron} makes an unbounded number of mistakes. However, there exists a linear classifier $\vec{w^*}=(1,0)$ which makes only one mistake in this scenario as shown in \Cref{fig:hing5}. 

\begin{figure}[ht!]
    \centering
    \begin{subfigure}[b]{0.43\textwidth}
        \centering
        \includegraphics[width=\textwidth]{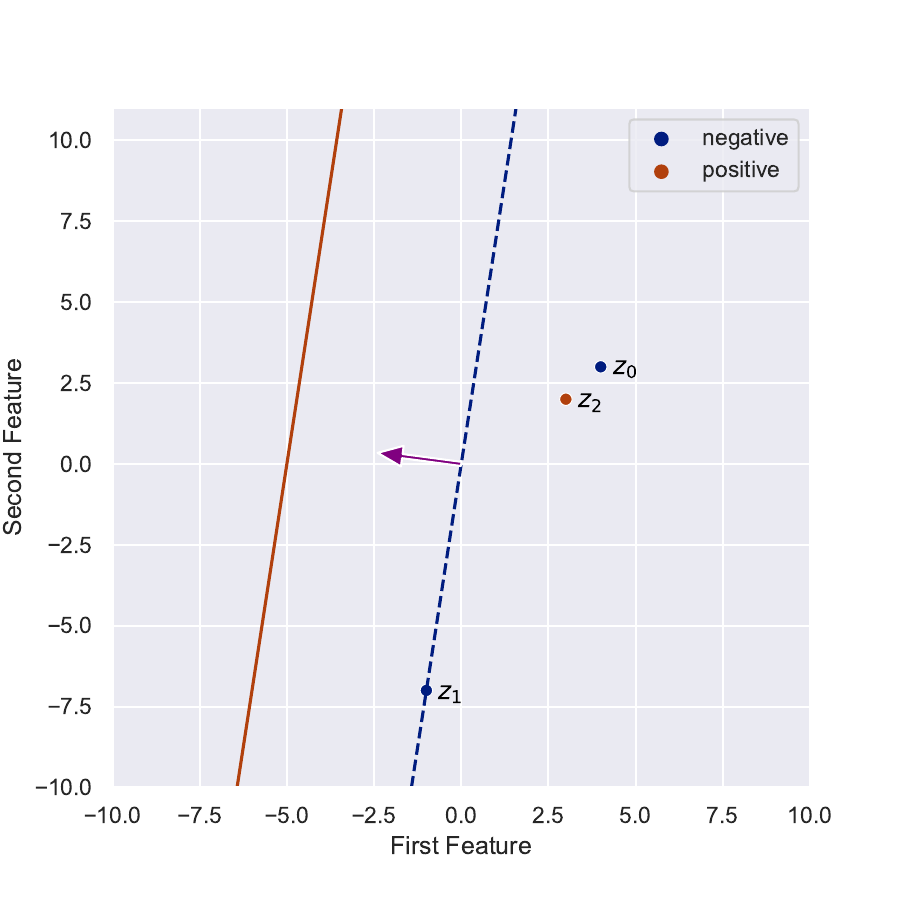}
       \caption{The algorithm makes a mistake on $\vec{z}_2=(3,2)$ when $\vec{w}=(-7,1)$.}
        \label{fig:hing2}
    \end{subfigure}
    \hfill
    \begin{subfigure}[b]{0.43\textwidth}
        \centering
        \includegraphics[width=\textwidth]{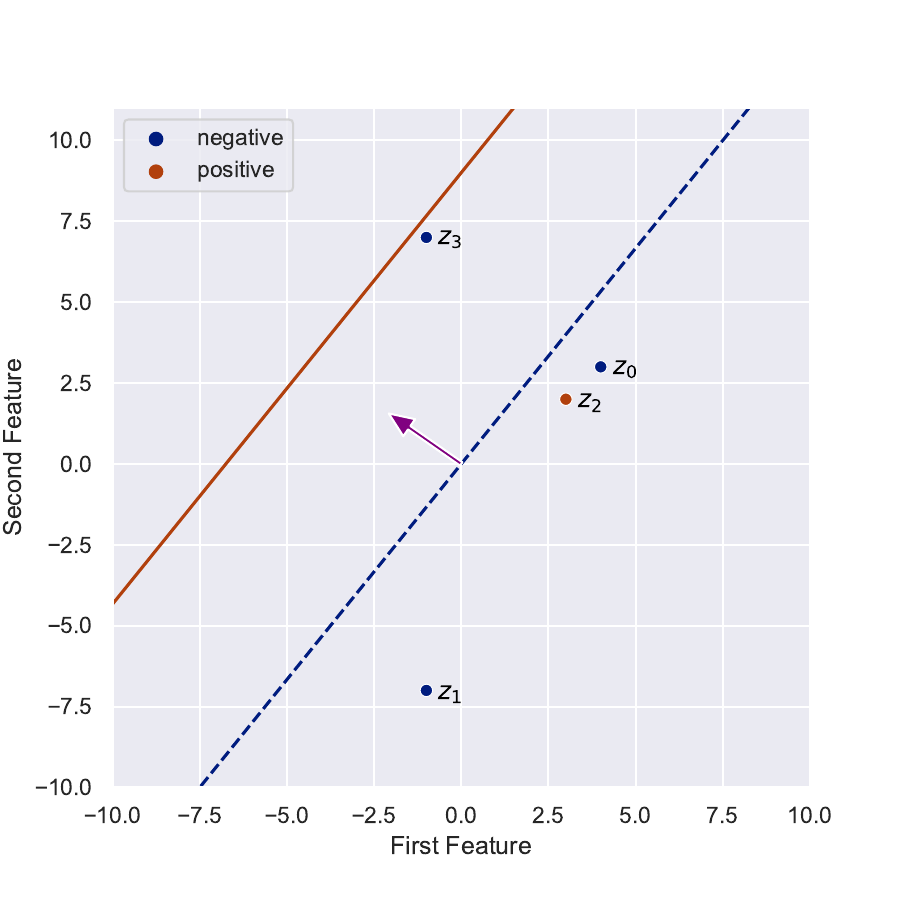}
        \caption{ $\vec{z}_{3}= (-1,7)$ is mistakenly classified when 
        $\vec{w}=(-4,3)$.}
        \label{fig:hing3}
    \end{subfigure}
        \begin{subfigure}[b]{0.43\textwidth}
        \centering
        \includegraphics[width=\textwidth]{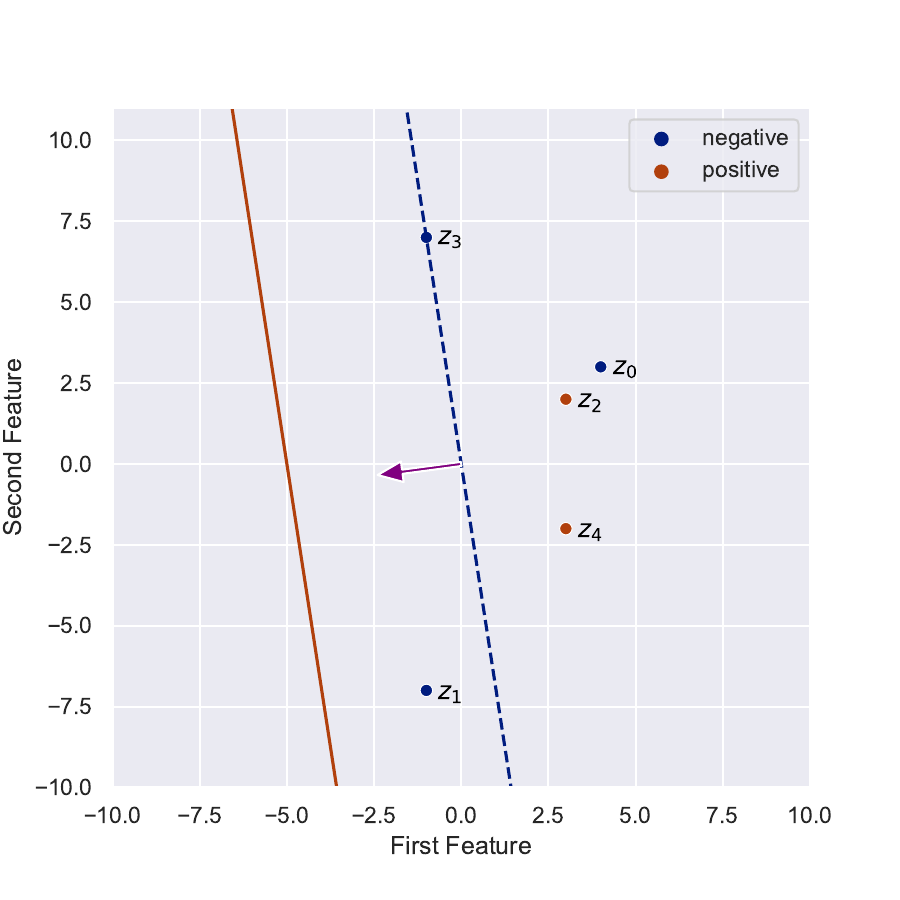}
       \caption{Positive example $\vec{z}_{4}= (3,-2)$ is misclassified when $\vec{w}=(-7,-1)$.}
        \label{fig:hing4}
    \end{subfigure}
    \hfill
    \begin{subfigure}[b]{0.43\textwidth}
        \centering
        \includegraphics[width=\textwidth]{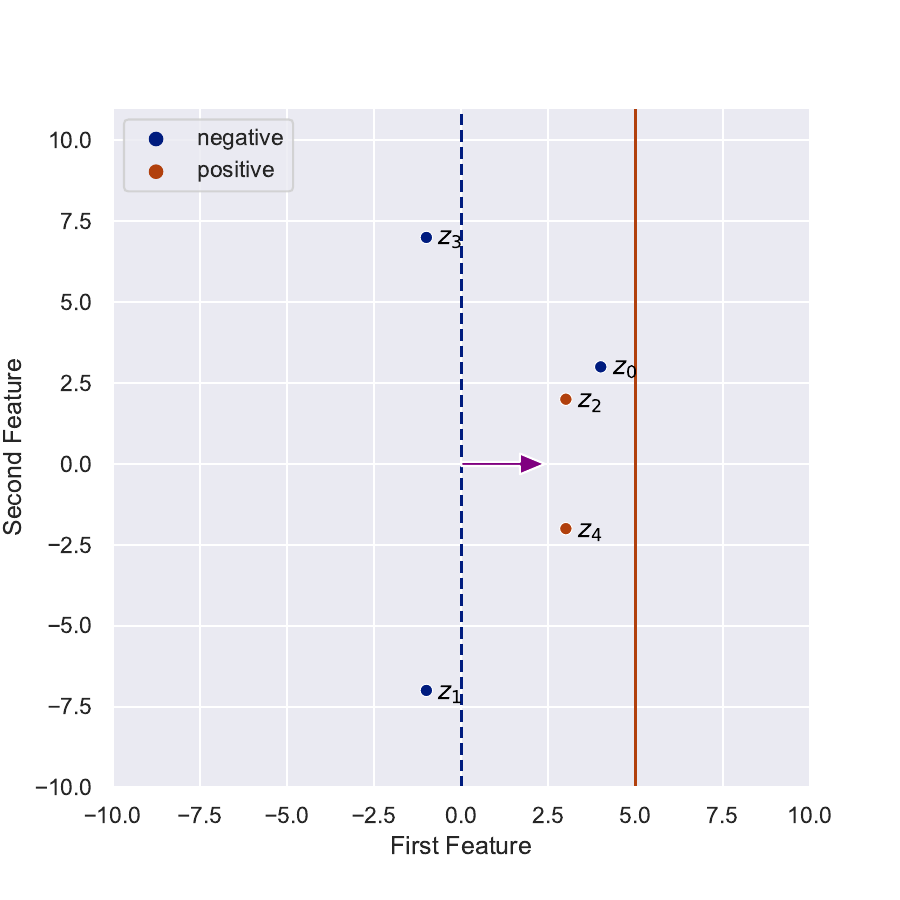}
        \caption{Data is not linearly separable, however, with $\vec{w^*}=(1,0)$, only one mistake is made.}
        \label{fig:hing5}
    \end{subfigure}
    \caption{\Cref{ex:hinge} shows \Cref{alg:strategic-Perceptron} makes an unbounded number of mistakes, where there is a separator with bounded hinge-loss. The dotted line in each figure shows the current $\vec{w}$, and the solid line shows the current classifier. The arrows show the positive direction of each classifier.}
\end{figure}
\end{example}

Third, when the separator of the original data points $\vec{z}_1,\vec{z}_2,\cdots$ does not cross the origin, as studied in \Cref{sec:bias}, \Cref{alg:strategic-Perceptron-bias} makes at most $O(R^3 |\vec{w^*}|^3)$ mistakes in the $\ell_2$ costs setting. Our last open problem is whether it is possible to improve the number of mistakes to $O(R^2 |\vec{w^*}|^2)$.

%% file: supplementary.tex
\section{Proof of \Cref{pr:num_updates_unknown_alpha}}
\label{appendix:proof-thm-num_updates_unknown_alpha}
\begin{proof}

Similar to the proof of \Cref{thm:num_updates}, we keep track of two quantities, $\vec{w}^T\vec{w^*}$, and $|\vec{w}|^2$. First, we show each time a mistake is made, $\vec{w}^T\vec{w^*}$ increases by at least $1/2$.
If we make a mistake on a positive example then,
\[(\vec{w}+\vec{\tilde{x}}_t)^T\vec{w^*} = \vec{w}^T\vec{w^*}+\vec{\tilde{x}}_t^T\vec{w^*}\geq \vec{w}^T\vec{w^*} + 1/2.\]
The last inequality holds by \Cref{lem:separability-unknown-alpha}. Similarly if we make a mistake on a negative example, then,
\[(\vec{w}-\vec{\tilde{x}}_t)^T\vec{w^*} = \vec{w}^T\vec{w^*}-\vec{\tilde{x}}_t^T\vec{w^*}\geq \vec{w}^T\vec{w^*} + 1/2.\]

Next, on each mistake we claim that $|\vec{w}|^2$ increases by at most $(R+\alpha'+\gamma/2)^2$. If we make a mistake on a positive example $\vec{x}_t$ then we have:
\[(\vec{w}+\vec{\tilde{x}}_t)^T(\vec{w}+\vec{\tilde{x}}_t) = |\vec{w}|^2+2\vec{\tilde{x}}_t^T\vec{w}+|\vec{\tilde{x}}_t|^2\leq |\vec{w}|^2+|\vec{\tilde{x}}_t|^2\leq |\vec{w}|^2+ (R+\alpha)^2.\]
The middle inequality is the result of applying \Cref{lem:dot_product_x_tilde_w_unknown}.
The last inequality comes from $R = \max_t|\vec{y}_t|$ implying $\max_t|\vec{\tilde{x}}_t| \leq R+\alpha$.

Similarly, if we make a mistake on a negative example $\vec{x}_t$, then we have,
\[(\vec{w}-\vec{\tilde{x}}_t)^T(\vec{w}-\vec{\tilde{x}}_t) = |\vec{w}|^2-2\vec{\tilde{x}}_t^T\vec{w}+|\vec{\tilde{x}}_t|^2\leq |\vec{w}|^2+|\vec{\tilde{x}}_t|^2\leq |\vec{w}|^2+ (R+\alpha)^2.\]
The inequalities hold similar to the previous case.

Therefore, if we make $M$ mistakes, then $\vec{w}^T\vec{w^*}\geq M/2$ and $|\vec{w}|^2\leq M(R+\alpha)^2$, or equivalently, $|\vec{w}|\leq (R+\alpha)\sqrt{M}$. Using the fact that $\vec{w}^T\vec{w^*}/|\vec{w^*}|\leq |\vec{w}|$, we have,
\begin{align*}
M/2|\vec{w^*}| \leq (R+\alpha)\sqrt{M} &\implies
\sqrt{M} \leq 2(R+\alpha)|\vec{w^*}|\\ 
\implies M \leq 4(R+\alpha)^2|\vec{w}^*|^2
&\implies M\leq 4(R+\alpha'+\gamma/2)^2|\vec{w}^*|^2.
\end{align*}
And the proof is complete.
\end{proof}

%% file: main.bbl
\begin{thebibliography}{10}

\bibitem{alon2020multiagent}
Tal Alon, Magdalen Dobson, Ariel Procaccia, Inbal Talgam-Cohen, and Jamie
  Tucker-Foltz.
\newblock Multiagent evaluation mechanisms.
\newblock {\em Proceedings of the AAAI Conference on Artificial Intelligence},
  34(02):1774--1781, Apr. 2020.

\bibitem{bechavod2020causal}
Yahav Bechavod, Katrina Ligett, Zhiwei~Steven Wu, and Juba Ziani.
\newblock Causal feature discovery through strategic modification.
\newblock {\em arXiv preprint arXiv:2002.07024}, 2020.

\bibitem{braverman2020role}
Mark Braverman and Sumegha Garg.
\newblock The role of randomness and noise in strategic classification.
\newblock In {\em 1st Symposium on Foundations of Responsible Computing, {FORC}
  2020}, volume 156 of {\em LIPIcs}, pages 9:1--9:20. Schloss Dagstuhl -
  Leibniz-Zentrum f{\"{u}}r Informatik, 2020.

\bibitem{10.1145/2020408.2020495}
Michael Br\"{u}ckner and Tobias Scheffer.
\newblock Stackelberg games for adversarial prediction problems.
\newblock In {\em Proceedings of the 17th ACM SIGKDD International Conference
  on Knowledge Discovery and Data Mining}, KDD ’11, page 547–555, New York,
  NY, USA, 2011. Association for Computing Machinery.

\bibitem{chen2020learning}
Yiling Chen, Yang Liu, and Chara Podimata.
\newblock Learning strategy-aware linear classifiers.
\newblock In {\em Proceedings of the Thirty-fourth Conference on Neural
  Information Processing Systems (NeurIPS 2020)}, 2020.

\bibitem{10.1145/3219166.3219193}
Jinshuo Dong, Aaron Roth, Zachary Schutzman, Bo~Waggoner, and Zhiwei~Steven Wu.
\newblock Strategic classification from revealed preferences.
\newblock In {\em Proceedings of the 2018 ACM Conference on Economics and
  Computation}, EC ’18, page 55–70, New York, NY, USA, 2018. Association
  for Computing Machinery.

\bibitem{haghtalab2020maximizing}
Nika Haghtalab, Nicole Immorlica, Brendan Lucier, and Jack~Z. Wang.
\newblock Maximizing welfare with incentive-aware evaluation mechanisms.
\newblock In Christian Bessiere, editor, {\em Proceedings of the Twenty-Ninth
  International Joint Conference on Artificial Intelligence, {IJCAI-20}}, pages
  160--166. International Joint Conferences on Artificial Intelligence
  Organization, 7 2020.
\newblock Main track.

\bibitem{Hardt2016}
Moritz Hardt, Nimrod Megiddo, Christos Papadimitriou, and Mary Wootters.
\newblock Strategic classification.
\newblock In {\em Proceedings of the 2016 ACM Conference on Innovations in
  Theoretical Computer Science}, ITCS ’16, page 111–122, New York, NY, USA,
  2016. Association for Computing Machinery.

\bibitem{Hu:2019:}
Lily Hu, Nicole Immorlica, and Jennifer~Wortman Vaughan.
\newblock The disparate effects of strategic manipulation.
\newblock In {\em Proceedings of the Conference on Fairness, Accountability,
  and Transparency}, FAT* '19, pages 259--268, New York, NY, USA, 2019. ACM.

\bibitem{Kleinberg2018HowDC}
Jon Kleinberg and Manish Raghavan.
\newblock How do classifiers induce agents to invest effort strategically?
\newblock EC '19, page 825–844, New York, NY, USA, 2019. Association for
  Computing Machinery.

\bibitem{miller2019strategic}
John Miller, Smitha Milli, and Moritz Hardt.
\newblock Strategic classification is causal modeling in disguise.
\newblock In Hal~Daumé III and Aarti Singh, editors, {\em Proceedings of the
  37th International Conference on Machine Learning}, volume 119 of {\em
  Proceedings of Machine Learning Research}, pages 6917--6926. PMLR, 13--18 Jul
  2020.

\bibitem{Milli2018TheSC}
Smitha Milli, John Miller, Anca~D. Dragan, and Moritz Hardt.
\newblock The social cost of strategic classification.
\newblock In {\em Proceedings of the Conference on Fairness, Accountability,
  and Transparency}, FAT* '19, page 230–239, New York, NY, USA, 2019.
  Association for Computing Machinery.

\bibitem{rosenblatt1958perceptron}
Frank Rosenblatt.
\newblock The perceptron: a probabilistic model for information storage and
  organization in the brain.
\newblock {\em Psychological review}, 65(6):386, 1958.

\bibitem{shavit2020learning}
Yonadav Shavit, Benjamin Edelman, and Brian Axelrod.
\newblock Causal strategic linear regression.
\newblock In Hal~Daumé III and Aarti Singh, editors, {\em Proceedings of the
  37th International Conference on Machine Learning}, volume 119 of {\em
  Proceedings of Machine Learning Research}, pages 8676--8686. PMLR, 13--18 Jul
  2020.

\end{thebibliography}
